\documentclass{article} 
\pdfoutput=1
\newif\ifarxiv\arxivtrue 

\usepackage[accepted]{rlj}

\usepackage{times}
\usepackage{booktabs}
\usepackage{array}
\usepackage{tabularx}
\usepackage{multirow}
\usepackage{makecell}
\usepackage[utf8]{inputenc} 
\usepackage[T1]{fontenc}    
\usepackage{threeparttable}
\usepackage{subcaption}

\usepackage{wrapfig}

\usepackage{algorithm}
\usepackage{algpseudocode}

\usepackage{amsmath}
\usepackage{amsthm}
\usepackage{amssymb}                
\usepackage{mathtools}              
\usepackage{mathrsfs}               
\mathtoolsset{showonlyrefs}         

\usepackage{lipsum}                 

\newtheorem{theorem}{Theorem}[section]

\newtheorem{lemma}[theorem]{Lemma}

\theoremstyle{definition}
\newtheorem{definition}[theorem]{Definition}

\theoremstyle{remark}
\newtheorem{remark}[theorem]{Remark}

\newtheorem{ass}{Assumption}[]

\newcommand{\thistheoremname}{}

\theoremstyle{plain}
\newtheorem*{genericthm*}{\thistheoremname}
\newenvironment{namedthm*}[1]
  {\renewcommand{\thistheoremname}{#1}%
   \begin{genericthm*}}
  {\end{genericthm*}}

\DeclareMathOperator*{\argmin}{arg\,min}
\DeclareMathOperator*{\argmax}{arg\,max}
\DeclareMathOperator*{\E}{\mathbb{E}}

\DeclareMathOperator*{\supp}{supp}

\usepackage{courier} 

\usepackage{graphicx} 
\usepackage{subcaption}
\usepackage[space]{grffile} 

\usepackage{hyperref}

\usepackage{xcolor}
\definecolor{dark-blue}{rgb}{0,0,0.7}
\definecolor{dark-green}{rgb}{0,0.6,0}
\definecolor{dark-red}{rgb}{0.8,0,0}
\definecolor{dark-yellow}{rgb}{0.8, 0.8, 0.0}
\definecolor{ppt-red}{HTML}{C00000}
\definecolor{ppt-blue}{HTML}{0F9ED5}
\hypersetup{
    colorlinks, linkcolor={dark-blue},
    citecolor={dark-blue}, urlcolor={dark-blue}
}

\usepackage{pifont}
\newcommand{\cmark}{\textcolor{dark-green}{\ding{51}}}
\newcommand{\xmark}{\textcolor{dark-red}{\ding{53}}}
\newcommand{\qmark}{\textcolor{gray}{{\boldmath$\sim$}}}

\usepackage{tabularx}

\title{Supervised Reward Inference}
\setrunningtitle{Supervised Reward Inference}

\author{Will Schwarzer\textsuperscript{1}, Jordan Schneider\textsuperscript{2,$\dagger$}, Philip S. Thomas\textsuperscript{1}, Scott Niekum\textsuperscript{1}}

\emails{wschwarzer@umass.edu}

\affiliations{
$^{1}$\textbf{University of Massachusetts}\\
$^{2}$\textbf{Anthropic}
\par
$^\dagger$ Work done while at the University of Texas.
}
\date{}

\contribution{
We formulate \emph{Supervised Reward Inference} (SRI), which uses supervised learning on multi-task state–reward labels to predict rewards from arbitrary, unassumed behavior.
}{
SRI requires training-time reward supervision and assumes the behavior distribution is fixed between training and deployment. Unlike inverse planning \citep{shah2019feasibility}, SRI is not restricted to tabular MDPs or explicit policies; unlike reward models supervised by language instructions or expert progress labels \citep[e.g.,][]{ma2023liv, yang2024rank2reward}, it infers the task from unlabeled, potentially suboptimal behavior.
}

\contribution{
We show that SRI with MSE loss asymptotically approaches Bayes-optimal apprenticeship learning for any (fixed but unknown) behavior model as its dataset size increases, subject to regularity, recognizability and state-coverage assumptions.
}{
This theorem is purely asymptotic, and it builds on known results in asymptotic statistics \citep{Sen2022gentle} and IRL theory \citep{Ramachandran2007bayesian}.
}

\contribution{
We develop an SRI architecture that 1) encodes trajectories once per task using a transformer \citep{vaswani2017attention} and a set transformer \citep{lee2019set}, and 2) predicts per-state reward with two small MLPs, enabling fast reward queries during policy optimization.
}{
This is one instantiation of the SRI objective; encoders can be swapped without changing the supervised formulation \citep{zaheer2017deepsets}.
}

\contribution{
Empirically, we show that SRI can yield strong policies in robotics tasks such as Meta-World reach and pick-place tasks even when given dramatically suboptimal behavior. When adapted to take the appropriate inputs, SRI also outperforms prior methods \citep{shah2019feasibility} in an existing tabular benchmark for reward inference from suboptimal behavior.
}{
Results use simulated behavior classes (not human behavior), 30 trials, and standard error intervals (no hypothesis tests). Meta-World experiments used single-task, optimality-assuming methods to establish naive performance baselines.
}

\summary{

\noindent Humans often indicate their goals through behavior that is suboptimal or not even attempting to optimize, such as gestures. We propose Supervised Reward Inference (SRI): given a multi-task dataset pairing behavior with \textit{known} rewards, SRI uses supervised learning to predict rewards directly from behavior, with no assumed behavior model. We prove that SRI is asymptotically Bayes-optimal, and show empirically that it achieves near-ceiling performance on a prior tabular benchmark and yields policies in Meta-World manipulation tasks that are as strong as the demonstrations allow, even when those demonstrations are dramatically suboptimal.
}


\begin{document}

\ifarxiv\else\makeCover\fi
\maketitle

\begin{abstract}
   \noindent Existing approaches to reward inference typically assume that humans provide demonstrations according to specific behavior models. However, humans often indicate their goals through a wide range of behaviors, from actions that are suboptimal due to poor planning or execution to behaviors intended to communicate goals rather than achieve them. One existing solution for inferring rewards from such behavior -- provided it is drawn from the same distribution at training and deployment -- is to construct a dataset of behavior paired with \textit{known} rewards, and to learn the mapping from behavior to rewards; however, prior methods in this family face notable limitations, such as restrictions to tabular settings. Given such a dataset, we propose instead that supervised learning offers a parsimonious yet powerful solution, which we term Supervised Reward Inference (SRI). Theoretically, we prove that SRI is asymptotically Bayes-optimal under standard assumptions. Empirically, SRI achieves near-ceiling performance on a prior benchmark for reward inference from suboptimal behavior, while on Meta-World robotics tasks, it infers rewards from even arbitrarily suboptimal demonstrations as accurately as those demonstrations allow. Finally, we demonstrate our framework's universality with straightforward generalizations to action- and goal-prediction.
\end{abstract}

\section{Introduction}
\label{sec:intro}

In order for artificial agents to achieve human goals, humans must first communicate their goals to the agents. While the traditional method of goal communication in reinforcement learning (RL) is explicit reward specification, specifying correct rewards can be challenging \citep{hadfieldmenell2017inverse, ratner2018reward, booth23perils, knox2024how}. 
This has highlighted the need for alternative modalities for reward specification, such as human demonstrations, the modality studied in inverse reinforcement learning (IRL) \citep{ng2000algorithms}.

IRL generally assumes that demonstrations are generated according to a specific model of human behavior, which ranges from noisy optimality \citep{ng2000algorithms, Ramachandran2007bayesian, zheng2014robust, chan2021human, laidlaw2022boltzmann, barnett2023active} to bounded reasoning \citep{evans2015learning, evans2016learning, zhi2020online} and beyond. 
While such models produce solvable learning problems, they are still far from accurate descriptions of the entirety of human behavior, for two reasons: 1) real-world human behavior demonstrates all of these suboptimalities at once, and many more not yet accounted for \citep{kahneman2011thinking, shah2019feasibility}; and 2) people frequently use entirely non-optimal behavior such as gestures in order to communicate goals.

In this paper, we investigate one approach to learning rewards from the full range of human behavior: framing a human's actions as an \textit{indication} of their goals, rather than an attempted \textit{optimization} of them. Similar to previous work \citep{enayati2018skill, reddy2018where, shah2019feasibility}, we assume access to a dataset of behaviors (e.g., demonstrations or gestures) and their associated ground-truth rewards.
However, rather than explicitly learning a behavior model that maps these rewards onto the behaviors, we use supervised learning to directly learn a mapping of behaviors onto rewards. This approach, which we call Supervised Reward Inference (SRI), \textbf{is an asymptotically Bayes-optimal approach for both reward inference and apprenticeship learning} (policy inference) for any fixed class of behavior under standard assumptions (Section \ref{sec:theory}); the behavior class must be fixed between training and deployment, but need not be known or modeled.

Surprisingly, our experiments also show that \textbf{SRI is a \textit{practical} method for learning rewards and policies} from a behavior--reward dataset. Despite using no inductive biases whatsoever about RL, our SRI implementation achieves near-ceiling performance on a known benchmark for reward inference from suboptimal behavior \citep{shah2019feasibility}, and demonstrates strong performance on Meta-World reach and pick-place tasks, even though our RL experiments showed that pick-place policy learning requires an exceptionally accurate reward function (Section \ref{sec:experiments}). These results support the idea that less-biased learning algorithms often outperform more-biased algorithms when provided with enough data \citep{sutton2019bitter, chen2021decision}; indeed, SRI's performance improves smoothly with data quantity, though aside from pick-place it can still perform well on challenging reward inference tasks even with limited data (Table \ref{tab:sri-data-efficiency}).

Finally, the supervised inference template underlying SRI extends naturally beyond rewards. In Appendix \ref{app:additionalresults}, we propose and evaluate Supervised Goal Inference (SGI) and Supervised Action Inference (SAI), variants of SRI that demonstrate strong performance when the available dataset instead pairs behavior with parameterizations of the reward function or optimal actions. More broadly, viewing behavior as an indication of intent rather than an optimization of it suggests a family of supervised approaches -- of which SRI, SGI and SAI are instances -- that replace behavior-model engineering with data. Because these methods are standard supervised learning, they can directly benefit from advances in representation learning and pretrained models (Section \ref{sec:conclusion}).

\section{Related Work}
\label{sec:relatedwork}

To the best of our knowledge, our work is the first to evaluate direct supervised learning for reward inference from unlabeled, arbitrarily suboptimal behavior.
Table~\ref{tab:ml_transposed_comparison_transposed} compares SRI to the most closely related method families; here, we review these and other prior work extending behavior imitation \citep{pomerleau1988alvinn} and reward inference \citep{ng2000algorithms} to use training data, account for various kinds of suboptimality, and learn from behavior classes other than demonstrations.

\textbf{Meta-IRL} Prior work has studied the use of multi-task demonstration datasets to improve the efficiency of IRL inference, under the names of multi-task IRL \citep{babes2011apprenticeship, dimitrakakis2012bayesian, choi2012nonparametric, gleave2018multitask}, meta-IRL \citep{xu2019learning, yu2019meta, ghasemipour2019smile, wang2021meta}, and lifelong IRL \citep{mendez2018lifelong}. Like SRI, such works allow fast, data-efficient reward inference, but do not directly enable reward inference from suboptimal demonstrations.

\textbf{Behavior model misspecification} \citet{armstrong2018occam} showed that it is generally impossible to simultaneously infer a demonstrator's reward function and their behavior model; thus, reward inference methods must either assume the behavior model (as do most methods) or learn it from data (as do SRI and \citealp{shah2019feasibility}). Later, \citet{Skalse2022MisspecificationII} showed theoretically that assuming a behavior model almost always produces incorrect reward functions when the model is incorrect. For example, assuming Boltzmann rationality only provides an asymptotically correct optimal policy set for ground-truth behavior models that take optimal actions most frequently.
Other works quantified the error induced by incorrect behavior models: \citet{shah2019feasibility} and \citet{chan2021human} showed that a variety of misspecifications can induce dramatically incorrect reward functions, while \citet{hong2023sensitivity} showed that in continuous-action MDPs, even arbitrarily small errors in the behavior model can result in almost arbitrarily large errors in the inferred reward parameters. 

\textbf{Learning from suboptimal demonstrations} Reward inference algorithms have been developed to account for a wide variety of suboptimalities, including hyperbolic discounting, myopia, false beliefs and bounded cognition \citep{evans2015learning, evans2016learning, zhi2020online}, autocorrelated action noise \citep{zheng2014robust}, mistaken transition models \citep{reddy2018where}, and risk-sensitive behavior \citep{singh2018risk, ratliff2020inverse}. This work also shows humans have little trouble inferring and accounting for each other's suboptimality \citep{evans2016learning, zhi2020online}.

\citet{shiarlis2016inverse} studied a setting where demonstrations are arbitrarily suboptimal, but are labeled as failures (to avoid) or successes; similarly, \citet{brown2020safe} used preferences over suboptimal demonstrations. \citet{brown2020better} and \citet{chen2021learning} extend this approach by using noise injection to automatically rank synthetic demonstrations to train a noise-averse reward model.

\textbf{Learning from general behavior} \citet{hadfieldmenell2016cooperative} and \citet{malik2018efficient} studied how to provide and learn from demonstrations selected according to their information content for the learner, rather than their accumulated reward. \citet{shah2019feasibility} developed the research direction closest to SRI. They present two algorithms for human behavior models and rewards, based on differentiable planners, in particular value iteration networks (VINs) \citep{tamar2016value}. Their first algorithm uses a setting similar to SRI's: it trains a value iteration network on a dataset of demonstration policies and corresponding reward functions to predict a policy given a reward function. At inference time, given a policy, the reward function is recovered through gradient descent. Unlike SRI, however, it has not been extended past tabular domains.

\textbf{Supervised and foundation-model reward learning} A growing body of robotics work learns reward models by direct supervision, either deriving progress or preference labels from demonstrations themselves \citep[e.g.,][]{yang2024rank2reward, tan2025robodopamine, liang2026robometer} or conditioning on a direct task specification (typically a language instruction) and obtaining rewards from vision-language alignment or foundation-model feedback \citep[e.g.,][]{ma2023liv, wang2024rlvlmf, alakuijala2024video, zhang2025rewind, zhai2025vlac, lee2026roboreward}. These methods share SRI's supervised template, but their task input is a direct specification: a language description, or demonstrations whose ordering or annotations supply the reward labels. SRI instead infers the task from a finite set of unlabeled trajectories that may be ambiguous and arbitrarily suboptimal, thus posing different technical problems and demanding a different architecture and training algorithm.

\begin{table*}
\caption{Comparison of several reward inference methods. Inverse planning (IP) refers to Algorithm 1 by \citet{shah2019feasibility}. Learning from observation (LfO) refers to inference from observation-only trajectories. The final three capabilities refer to inferring at least Bayes-optimal rewards from demonstrations generated according to the specified type of behavior; for the latter two types, this assumes access to a dataset of demonstrations drawn from the same distribution (including the same behavior model) as the inference demonstrations. For CIRL \citep{hadfieldmenell2016cooperative}, we assume a non-interactive CIRL game with one demonstration and one deployment phase. \qmark~indicates that a framework may in principle support a capability that has not yet been demonstrated.
  }
  \centering
  \begin{tabularx}{\textwidth}{>{\hsize=2\hsize}X | >{\centering\arraybackslash\hsize=0.7\hsize}X>{\centering\arraybackslash\hsize=0.7\hsize}X>{\centering\arraybackslash\hsize=0.7\hsize}X>{\centering\arraybackslash\hsize=0.7\hsize}X | >{\centering\arraybackslash\hsize=0.7\hsize}X}
    \toprule
    Capability & IRL & Meta IRL & CIRL & IP & \textbf{SRI} \\
    \midrule
    No reward dataset required & \cmark & \cmark & \cmark & \xmark & \xmark \\
    Few-shot inference & \xmark & \cmark & \xmark & \xmark & \cmark \\
    Simulator-free inference & \xmark  & \cmark & \xmark & \cmark & \cmark \\
    LfO possible & \cmark & \cmark & \cmark & \xmark & \cmark \\
    Continuous MDPs & \cmark & \cmark & \qmark & \xmark & \cmark \\
    Known suboptimal behavior & \cmark & \cmark & \cmark & \cmark & \cmark \\
    Lim. suboptimal behavior & \xmark & \xmark & \xmark & \cmark & \cmark \\
    Arbitrary behavior & \xmark & \xmark & \xmark & \qmark & \cmark \\
    \bottomrule
  \end{tabularx}
  \label{tab:ml_transposed_comparison_transposed}
\end{table*}

\section{Background}
\label{sec:background}

While SRI does not assume that human behavior is generated by any particular learning or control algorithm, we still formalize the notions of ``behavior'' and ``goals'' using notation from reinforcement learning (RL) \citep{SuttonBarto} and inverse reinforcement learning (IRL) \citep{ng2000algorithms}, which we review here.

Control problems studied in RL are formalized mathematically as a Markov decision process (MDP). An MDP $\mathcal{M} = (\mathcal{S}, \mathcal{A}, p, r, d_0, \gamma)$ consists of possibly infinite sets of states, $\mathcal{S}$, and actions, $\mathcal{A}$; a transition function $p: \mathcal{S}\times \mathcal{A} \times \mathcal{S}  \rightarrow [0, 1]$; a reward function $r: \mathcal{S} \times \mathcal{A} \rightarrow \mathbb{R}$; an initial state distribution $d_0: \mathcal{S} \rightarrow [0, 1]$; and a discount factor $\gamma \in [0, 1]$.

A policy $\pi: \mathcal{S} \times \mathcal{A} \rightarrow [0, 1]$ is a function describing the agent's probability of selecting an action in any given state; let $\Pi$ be the set of all policies. Policies interact with an MDP to produce stochastic processes known as episodes: $(S_0, A_0, R_0, S_1, A_1, R_1, \dots)$ such that $S_0 \sim d_0$, $A_t \sim \pi(S_t, \cdot)$, $R_t = r(S_t, A_t)$, and $S_{t+1} \sim p(S_t, A_t, \cdot)$. 
MDPs can also be partially observable, meaning they also have a set of observations $\mathcal{O}$ and an emission function $\Omega: \mathcal{S} \times \mathcal{O} \rightarrow [0, 1]$. In this case, episodes include observations generated by the emission function, $O_i \sim \Omega(S_i, \cdot)$, and the policy $\pi: \mathcal{O} \times \mathcal{A} \rightarrow [0, 1]$ instead maps observations to distributions over actions: $A_i \sim \pi(O_i, \cdot)$.

In the notation used in this paper, we assume for simplicity that the reward function can be described as a function of state alone: $\left(R_t \perp\!\!\!\perp A_t | S_t\right)$. Thus, we will write $r(S_t)$ for brevity. Such state-based rewards are common in goal-based robotic manipulation tasks, for example. However, our methods apply equally well to the case where actions influence the reward.

The objective of an RL agent in an MDP is to accumulate as much reward as possible, subject to exponential time discounting. Formally, the discounted return starting at time $t$ is the sum of rewards at and after $t$, discounted exponentially by $\gamma$: $G_t = \sum_{i=0}^\infty \gamma^{i}R_{t+i}$. The expected discounted episodic return in an MDP is the expected value of $G_0$ for a given policy: $J(\pi)=\E[G_0 ; \pi]$, where semicolon $\pi$ indicates that $A_t \sim \pi(S_t, \cdot)$. RL in a given MDP is thus the optimization problem $\argmax_\pi J(\pi)$. Let $\pi^*$ be such an optimal policy, and let its expected return be $J^*$.

\textbf{Reward Inference.}
In reward inference problems such as SRI, IRL, or CIRL, the reward function is unknown to the agent, but typically train-time access to the underlying reward-free MDP, $\mathcal{M} \setminus \{r\}$, is still assumed.\footnote{In this paper, we use ``reward inference'' to denote reward inference from behavior, such as IRL.} In place of the reward, some number of trajectories in the environment are provided, consisting of sequences of either observations, states, or states and actions.\footnote{This setup describes a two-phase CIRL game; CIRL allows for multiple learning-deployment interactions between agent and demonstrator, but such an interactive problem setup is beyond the scope of this paper.}
In settings where the human behaves roughly optimally for the task they intend the imitator to complete, these trajectories are called \textit{demonstrations}, but for generality we call them ``behavior trajectories''. In this paper, we will assume the most difficult setting, where trajectories are sequences of observations of length $L_B$: $\tau = (o_0, o_1, \dots, o_{L_B}) \in \mathcal{T}$, where $\mathcal{T} := (\mathcal{O})^{L_B}$.
Thus, to indicate a single task, the agent is provided with $\left\{\tau_n\right\}_{n=1}^N \in \mathcal{T}^N$.

The agent's goal is to use these trajectories to infer the reward function. The reward function itself is sometimes the final output, but our focus is on optimizing the inferred reward function and evaluating the resulting policy against the hidden ground truth reward function. 

\subsection{Learning Behavior Models from Known Rewards}
\label{sec:learningfromknownrewards}

Reward inference models traditionally infer a completely unknown reward function $r$ by assuming that the trajectories $\{\tau_n\}$ are generated according to a specific, known behavior model $b: \mathcal{R} \rightarrow \Pi$, where $\mathcal{R}$ is the space of reward functions \citep{Ramachandran2007bayesian, Skalse2022MisspecificationII}. 

However,
inferring rewards from behavior alone is fundamentally underdetermined unless one assumes a correct mapping from rewards to behavior (a planner/behavior model): the same observed behavior can typically be explained by many reward--planner pairs, and simplicity biases are insufficient to recover the agent’s ``true'' preferences in general \citep{armstrong2018occam}.
Related no-free-lunch results from the corrupted-reward literature show that when the channel from latent reward to observed signals is unconstrained, any fixed algorithm can incur large regret in at least one observationally indistinguishable world \citep{everitt2017reinforcement}.
Even when committing to a particular behavioral model class, inverse reward inference can be highly brittle to model misspecification, including in continuous-control settings \citep{Skalse2022MisspecificationII,hong2023sensitivity}.
Consequently, behavior-only methods inevitably rely on strong assumptions about the reward-to-behavior map in order to demonstrate reasonable performance.

A pragmatic solution is to introduce reward supervision by collecting behavior on tasks with known or labelable rewards, which can be used either to fit a behavior model (as in prior work) or, as we propose, to directly learn the inverse mapping from behavior to rewards. For example, \citet{reddy2018where}, \citet{enayati2018skill}, \citet{carrenomedrano2019incrementalestimationofusersexpertiselevel}, and \citet{ghosal2023effectmodelinghumanrationalitylevel} use demonstrations by humans in tasks with known rewards to infer parameters of their behavior models, such as their Boltzmann-rationality temperature parameter or their internal beliefs about the transition model, $p$; similarly, \citet{milliken2017modelinguserexpertiseforchoosinglevelsofsharedautonomy} estimate a human's expertise in a driving task using the knowledge that hitting obstacles is undesirable. Finally, in their first algorithm, \citet{shah2019feasibility} use a dataset of known reward functions and known policies to learn any behavior model that can be produced by a value iteration network in a tabular MDP.

\section{Supervised Reward Inference}
\label{sec:sri}

We propose a simpler approach for performing reward inference using a dataset of human behavior for known rewards. Rather than training a parameterized behavior model from data, we simply train the inverse model to directly map trajectories to reward functions. 

This direct reward function inference approach, if it performed $N$-shot inference (i.e., used $N$ trajectories as input), would produce a model from trajectories to reward functions \(f_\theta: \mathcal{T}^N \rightarrow \mathcal{R}\), 
allowing reward inference on a single state $s$ through
\(f_\theta(\{\tau_n\}_{n=1}^N)(s).\) 

Such an $f_\theta$ would be applicable in those cases studied previously where the exact reward function is known \citep{shah2019feasibility}, or where a parameterized form $r_\psi$ of the reward function is known, in which case $f_\theta$ could use $\psi$ as its target. However, this is not a general solution, as true human reward functions in complex environments are unlikely to have known parameterizations (see Section \ref{sec:datasetconstruction}).

Instead, we teach $f_\theta$
to predict samples of the reward given a state as input: $r(s) \approx f_\theta(\{\tau_n\}_{n=1}^N, s)$. A further enhancement offers an immense efficiency gain: the behavior trajectories (and thus task) 
need not be reprocessed at every timestep, and can instead 
be preprocessed into a task encoding. We call the resulting task encoder $f_{\theta_f}$ (the \textcolor{ppt-blue}{blue path} in Figure \ref{fig:arch}), and the state-encoder and final reward model $g_{\theta_g}$ (the \textcolor{ppt-red}{red path} in Figure \ref{fig:arch}).\footnote{Note that this structure also allows us to train multi-task policies by conditioning on the task embedding output from $f_{\theta_f}$ \citep{yu2019meta}, which we explore in experiments with reach tasks.} This final structure allows us to formally define SRI.

\begin{figure}
  \centering
  \begin{minipage}[t]{\linewidth}
    \captionsetup{type=algorithm}      
    \begin{algorithm}[H]
    \caption{SRI Training with Gradient Descent}
    \label{alg:sri}
    \begin{algorithmic}[1] 
        \State {\bfseries Input:} Num tasks $K$, num dems/task $N_T$, num state-reward samples/task $N_s$, batch size $M$, num inference dems $N_I$, learning rate $\alpha$, and dataset $\mathcal{D} = \{(\{\tau_{k,i}\}_{i=1}^{N_T}, \{(s_{k,j}, r_{k,j})\}_{j=1}^{N_s})\}_{k=1}^K$.
        \State Initialize encoder $f_{\theta_f}: \mathcal{T}^{N_I} \rightarrow \mathbb{R}^d$ and reward model $g_{\theta_g}: \mathcal{S} \times \mathbb{R}^d \rightarrow \mathbb{R}$. Let $\theta = (\theta_f, \theta_g)$.
        \Repeat
            \State Sample batch $\{(\mathcal{T}_k, \mathcal{S}_k)\}_{k=1}^M \sim \mathcal{D}$, where $\mathcal{T}_k=\{\tau_{k,i}\}_{i=1}^{N_T}$ are trajectories and $\mathcal{S}_k=\{(s_{k,j}, r_{k,j})\}_{j=1}^{N_s}$ are state-reward pairs for task $k$.
            \State For each task $k \in [1,M]$, randomly select $N_I$ trajectories $\mathcal{T}'_k \subseteq \mathcal{T}_k$.
            \State Compute task embeddings $(\psi_k)_{k=1}^M = (f_{\theta_f}(\mathcal{T}'_k))_{k=1}^M$.
            \State Predict rewards $(\hat{r}_{k,j})_{k \in [1,M], j \in [1,N_s]} = (g_{\theta_g}(s_{k,j}, \psi_k))_{k \in [1,M], j \in [1,N_s]}$.
            \State Compute loss $\mathcal{L}(\theta) \leftarrow \frac{1}{MN_s}\sum_{k=1}^M \sum_{j=1}^{N_s} (\hat{r}_{k,j} - r_{k,j})^2$.
            \State Update parameters $\theta \leftarrow \theta - \alpha \nabla_{\theta} \mathcal{L}(\theta)$.
        \Until{convergence}
        \State {\bfseries Output:} Learned parameters $\theta = (\theta_f, \theta_g)$.
    \end{algorithmic}
    \end{algorithm}
  \end{minipage}
\end{figure}



\begin{definition}[Supervised Reward Inference]

Given: \textbf{a}) a random set of behavior trajectories $\{T_n\}_{n=1}^N \in \mathcal{T}^N$ and a random reward function $R$ jointly following a distribution $\mathcal{D}_T$; \textbf{b}) a random state $S$ following some distribution $\mathcal{D}_S$; \textbf{c}) 
parameterized function families $f_{\theta_f}: \mathcal{T}^N \rightarrow \Psi$ and $g_{\theta_g}: \mathcal{S} \times \Psi \rightarrow \mathbb{R}$, where $\Psi$ is a task embedding space; and \textbf{d}) a regression loss function $\mathcal{L}: \mathbb{R} \times \mathbb{R} \rightarrow \mathbb{R}$, \textit{supervised reward inference} (SRI) is the following \textbf{multi-task} minimization problem:
\begin{equation} 
    \argmin_{\theta_f, \theta_g}\mathbb{E}\bigg[\mathcal{L}\bigg (g_{\theta_g}\big(S, f_{\theta_f}(\{T_n\}_{n=1}^N) \big), R(S)\bigg) \bigg ]. \label{sri}
\end{equation}
\end{definition}
 
See Algorithm \ref{alg:sri} for an example gradient-descent-based implementation. Note that the state samples that are labeled with rewards do not need to be taken from the behavior trajectories. Indeed, it is often best for states and behaviors to be separate: the state samples should be representative of the state space optimized during RL, but the behaviors need not be. Note that the distribution $\mathcal{D}_S$ is a design choice: SRI will learn to predict rewards accurately on states drawn from $\mathcal{D}_S$, so it should be chosen to cover the states visited by downstream RL.

\subsection{Training Data}
\label{sec:datasetconstruction}
SRI assumes a task-indexed supervised dataset. For each task $k$, the training data contain (i)~a set of behavior trajectories $\{\tau_{k,i}\}$ and (ii)~a set of reward-labeled states $\{(s_{k,j}, r_{k,j})\}$. These two sets need only correspond to the same task; the reward-labeled states need not be states visited in the trajectories themselves. This separation is useful because the trajectories identify the task, while the labeled states can be sampled independently to cover the part of the state space on which the downstream RL policy will later optimize. As discussed in Section~\ref{sec:learningfromknownrewards}, one natural source of such data is behavior collected under known reward functions, where many more state-reward labels are comparatively cheap once task rewards are available; our experiments follow this approach (see Tables~\ref{tab:reachdata} and~\ref{tab:ppdata}). More generally, SRI replaces assumptions about a fixed reward-to-behavior map with an assumption that such task-indexed reward supervision is available.

\begin{wrapfigure}[18]{r}{0.48\textwidth}
\centering
    \includegraphics[width=\linewidth]{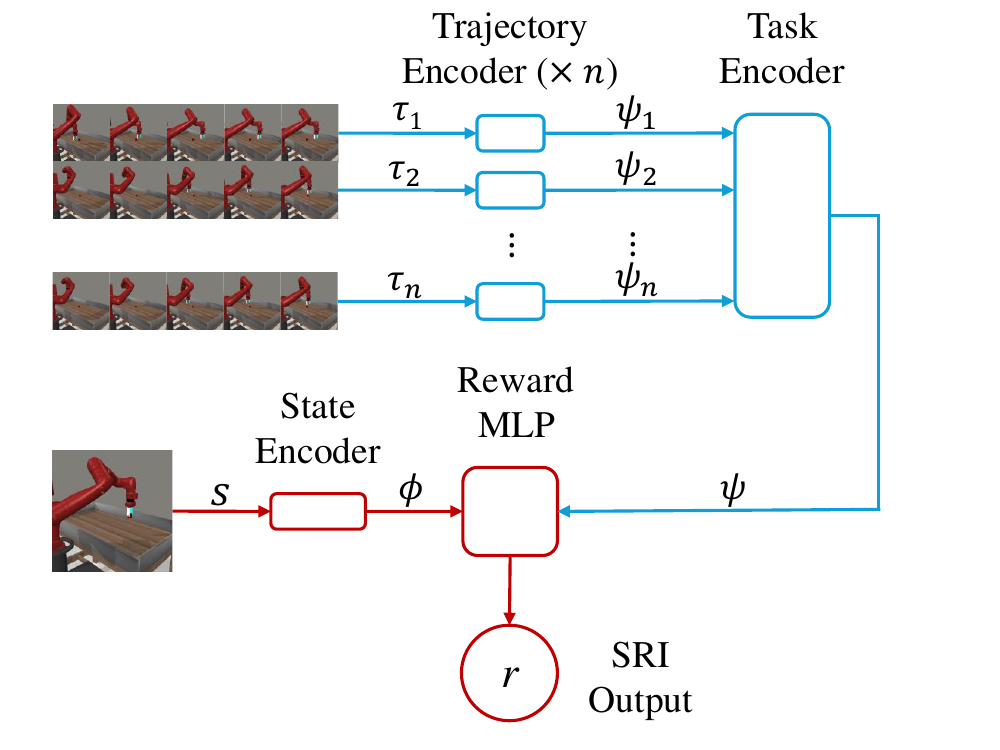}
    \captionof{figure}{%
      Trajectories are encoded with a transformer and pooled by a set
      transformer into a task embedding $\psi$ (\textcolor{ppt-blue}{blue}, once per task);
      each state is mapped to an observation embedding $\phi$ by two
      small MLPs (\textcolor{ppt-red}{red}, per timestep); a final MLP combines
      $(\psi,\phi)$ to predict reward.}
    \label{fig:arch}
\end{wrapfigure}

\subsection{Example Architecture}
\label{sec:architecture}

Figure \ref{fig:arch} demonstrates the abstract structure of the architecture that we used for SRI in our experiments; see Appendix \ref{app:arch} for exact details. All $g_{\theta_g}$ networks used (red path) were MLPs, while the trajectory encoder was a transformer \citep{vaswani2017attention}, and the task encoder was a set transformer \citep{lee2019set}.

\section{Theoretical Results}
\label{sec:theory}



Are there any classes of behavior which are too suboptimal or arbitrary even for SRI? For which classes of reward inference problems will SRI produce an optimal reward model? In this section, we state the theoretical answer to these questions: in the data limit, SRI approaches Bayes optimality for \textit{any} reward inference problem as long as the problem and SRI's model family satisfy certain compactness and niceness assumptions. (SRI requires no information about the \textit{test-time} task besides behavior trajectories, which are assumed to be drawn from the same distribution as the training trajectories.)
Concretely, in Appendix \ref{app:theory}, using the Bayesian inverse reinforcement learning framework \citep{Ramachandran2007bayesian}, we formally state and prove the following theorem.

\begin{namedthm*}{Main Theorem, Paraphrased}
[Asymptotic Optimality of SRI Algorithms]
\label{thm:policy_convergence_main}
Consider an SRI problem in an MDP with jointly distributed reward $R$ and trajectory samples $T \coloneqq \{T_n\}_{n=1}^N$, with marginal distribution $P_T$ over $T$. Suppose that the following assumptions hold.

\textit{Assumptions:} \textbf{1)} The MDP has compact state and action spaces $\mathcal{S}$ and $\mathcal{A}$ and bounded returns (with either $\gamma<1$ or finite-horizon truncation); \textbf{2)} The MDP has a random (unknown) continuous reward function $R$;
\textbf{3)} SRI's model family is $\{f_\theta\}_{\theta \in \Theta}$ for compact $\Theta$;
\textbf{4)} $\{f_\theta\}$ is equicontinuous;
\textbf{5)} SRI globally minimizes its mean-squared-error loss, and $\{f_\theta\}$ contains a continuous version of the posterior-mean regression function on $\mathcal S\times\mathcal X$, where $\mathcal X\coloneq \supp(P_T)$;
\textbf{6)} We sample trajectories and rewards from the true distribution, and states from a distribution with full support on $\mathcal{S}$.

\textbf{Claim:} Any SRI algorithm in this setting approaches asymptotic Bayes-optimality in two senses as the dataset grows: first, its inferred reward functions almost surely converge uniformly on \(\mathcal S\times\mathcal X\) to the expectation of the posterior distribution of \(R\) given \(T\); second, the returns of its optimal policies converge almost surely to the maximum expected return under this posterior distribution.

\end{namedthm*}

\begin{proof}[Proof sketch]
The proof outlines three main steps:
1) Uniform Reward Convergence: SRI's inferred reward converges uniformly to the expected posterior reward.
2) Optimality of the Mean: Maximizing return under this mean is equivalent to maximizing expected return under the posterior.
3) Policy Convergence: Uniform reward convergence implies that the optimal policies for SRI's reward functions achieve optimal performance. Note that all convergence is almost sure.

Step 1: \textbf{Reward Convergence}. First, pointwise convergence is shown. Boundedness (of rewards and function family) implies a Glivenko-Cantelli loss, ensuring uniform convergence across \textit{parameters} to an optimal $\theta$. Then, equicontinuity of $\{f_\theta\}$ and the Arzel\`a-Ascoli theorem provide uniform convergence across \textit{states and on-support trajectory inputs} for a reward function subsequence. This extends to the full sequence, proving uniform parameter-wise and input-wise convergence to the expected posterior reward.

Step 2: \textbf{Optimality of the Mean}. We generalize a result by \citet{Ramachandran2007bayesian}. Using Fubini's Theorem and Bounded Convergence, we show maximizing return under the expected posterior reward equals maximizing expected return under the full posterior.

Step 3: \textbf{Policy Convergence}. Bounded MDP returns and uniform reward convergence (from Step 1) allow bounding the deviation of SRI's predicted returns across policies. This return deviation shrinks with more data. Therefore, SRI's optimal policies approach optimality with respect to the expected posterior reward, achieving expected optimality under the full reward posterior.
\end{proof}

\section{Experiments}
\label{sec:experiments}

In the previous section, we showed that, theoretically, an ideal SRI algorithm can solve reward inference as well as it is possible to solve it given the available behavior.
In this section, we evaluate that claim in two complementary settings.
First, in Meta-World, we test whether trajectory-input SRI can infer useful rewards in continuous-control robotic manipulation tasks from highly suboptimal behavior -- including gestures that only indirectly indicate goals -- and how its performance varies with demonstration quality and data quantity.
Second, on the tabular gridworld benchmark of \citet{shah2019feasibility}, we compare SRI directly against the closest prior arbitrary-behavior reward-inference methods in the setting for which they were originally designed.
Code is \href{https://github.com/willschwarzer/supervised-reward-inference}{publicly available}.


\subsection{Meta-World Experiment Design}
\label{sec:tasks}

\textbf{Tasks.} We use Meta-World reach and pick-place tasks \citep{yu2019metaworld} with randomly distributed goals. Pick-place is substantially harder: it requires SRI to infer precise shaping rewards for grasping in addition to a large, discontinuous success reward (see Appendix~\ref{app:experimentaldetails} for full details).

\textbf{Demonstrations.} We designed five demonstration classes to test SRI under qualitatively different suboptimalities. \textsc{Noisy}$_\varepsilon$ demonstrations test graceful degradation: the end-effector reaches toward the goal but takes a random action with probability $\varepsilon$ each timestep. \textsc{Psychic}$_\alpha$ demonstrations test systematically \textit{wrong} behavior: the end-effector reaches deterministically toward an incorrect position offset from the goal ($\alpha\!=\!1.0$: no offset; $\alpha\!=\!-1.0$: mirrored through the origin). \textsc{Hard} demonstrations combine both challenges: the end-effector circles around an incorrect position. \textsc{Gesture} demonstrations test non-demonstration behavior on pick-place tasks: the end-effector merely reaches toward the goal without ever picking or placing. See Appendix~\ref{app:demonstrationdetails} for exact specifications.

\textbf{Data and training.} Except where otherwise noted, SRI received 1,280 tasks, each with 100 demonstrations and 10,000 randomly sampled state-reward pairs (see Tables~\ref{tab:reachdata} and~\ref{tab:ppdata} for data efficiency results). Policies were trained for 5 million (reach) or 10 million (pick-place) environment steps with TQC \citep{kuznetsov2020controllingoverestimationbiaswithtruncatedmixtureofcontinuousdistributionalquantilecritics} for reach and PPO \citep{schulman2017proximalpolicyoptimizationalgorithms} for pick-place, via Stable-Baselines3 \citep{stable-baselines3}. See Appendix~\ref{app:sritrainingdetails} for SRI and RL hyperparameters.

\textbf{Metrics and baselines.}
\label{sec:metrics}
Performance is measured by average normalized goal proximity ($\uparrow$): 1 minus the distance to the goal, scaled so the initial distance is 1. All experiments use 30 trials with standard error bars. We compare against three single-task imitation learning baselines -- BC \citep{pomerleau1988alvinn}, GAIL \citep{ho2016generative}, and AIRL \citep{fu2018learning} -- and against policies trained with ground-truth rewards (GT-RL). See Appendix~\ref{app:metrics} for additional details.

\subsection{Meta-World Results}
\label{sec:results}


\textbf{Reward inference from gestures.} On pick-place tasks, SRI inferred reward functions from reach-toward-goal gestures that never actually picked or placed anything, achieving $0.822 \pm 0.051$ normalized goal proximity\label{tab:pickplace} (GT-RL: $0.903 \pm 0.011$; all imitation learning baselines: $\leq 0$; 30 trials).

\begin{wrapfigure}[23]{R}{0.45\textwidth}
    \centering
    \includegraphics[width=0.44\textwidth]{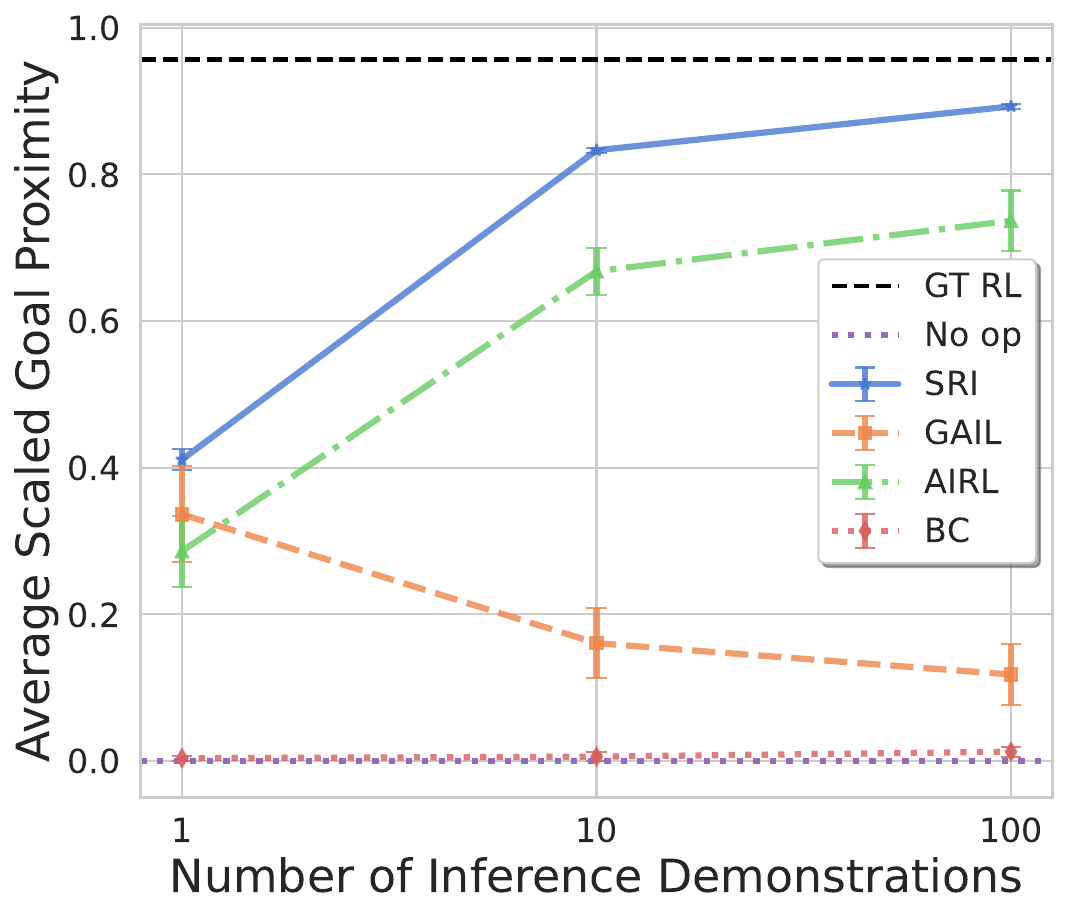}
    \caption{\textbf{Demonstration efficiency.} Mean normalized goal proximity of SRI and imitation baselines on Meta‑World reach tasks using \textsc{noisy}$_{0.87}$ demonstrations, plotted versus the number of demos per task. Error bars show $\pm$ standard error over 30 trials (Sec.~\ref{sec:metrics}). SRI's training makes it demonstration-efficient.}
    \label{fig:reach_dem_efficiency}
\end{wrapfigure}

\textbf{Robustness to suboptimal demonstrations.} More broadly, SRI learned accurate reward functions whenever the demonstrations identified the goal, including in a wide variety of situations where optimality-assuming methods completely fail (Figure \ref{fig:offset_reach}; Tables \ref{tab:reachdata}, \ref{tab:ppdata}). It performed well even when the behavior-reward mapping was profoundly noisy (Figure \ref{fig:noisy_reach}), and even when it only received a single noisy demonstration (Figure \ref{fig:reach_dem_efficiency}); however, it performed substantially better when the behavior-reward mapping was arbitrarily suboptimal but still invertible (Figure \ref{fig:offset_reach}). Finally, while challenging reward functions still require a substantial number of labeled observations to be learned accurately (Table \ref{tab:ppdata}), potentially necessitating the use of self-supervised learning methods, simpler tasks may be solvable with quantities of data small enough to collect manually (Table \ref{tab:reachdata}).

\begin{figure}[t]
  \centering
  \begin{subfigure}{.48\linewidth}
      \includegraphics[width=\linewidth]{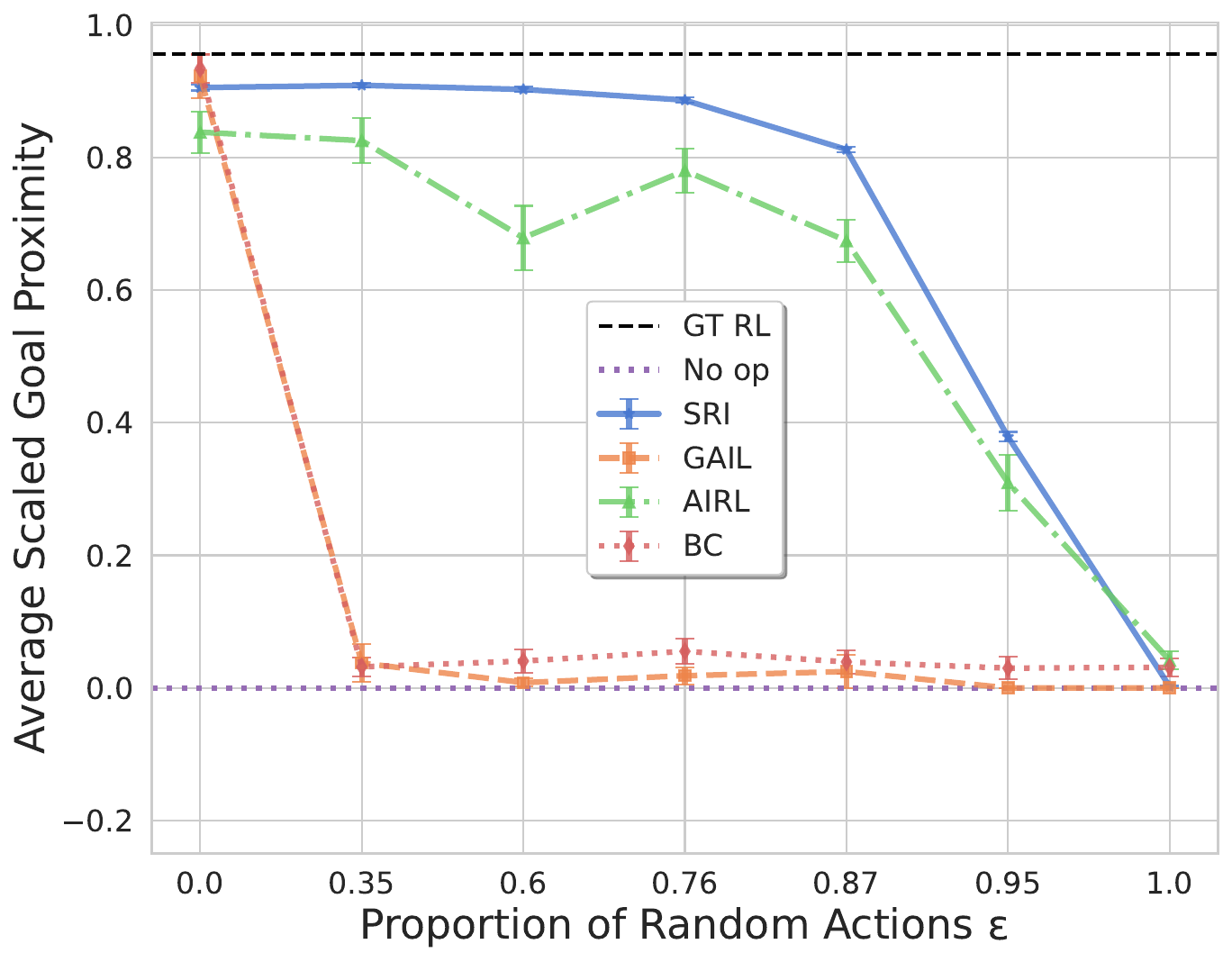}%
      \caption{\textbf{Random actions}: \textsc{noisy}$_\varepsilon$}
      \label{fig:noisy_reach}
    \end{subfigure}
  \hfill
  \begin{subfigure}{.48\linewidth}
      \includegraphics[width=\linewidth]{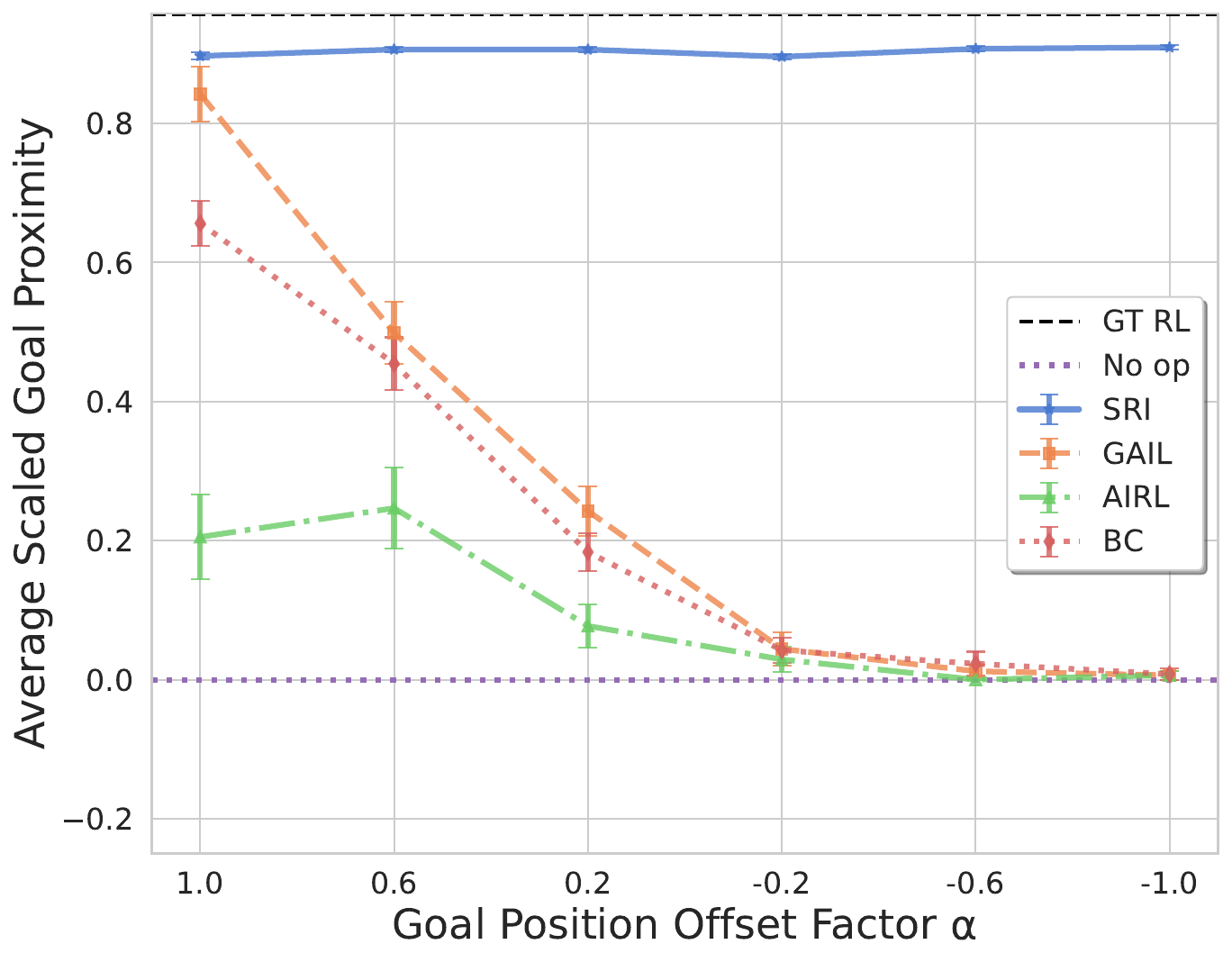}%
      \caption{\textbf{Reach to offset goal}: \textsc{psychic}$_\alpha$}
      \label{fig:offset_reach}
      \end{subfigure}
  \caption{\textbf{Robustness of SRI to suboptimal demonstrations in reach tasks.}
           Error bars: standard error over 30 trials. SRI infers near-optimal policies given \textit{invertible} suboptimalities (\textbf{b}).}
  \label{fig:reach_robustness}
\end{figure}


\begin{table}[b]
  \centering
  \footnotesize
  \caption{Data efficiency of SRI on two Meta-World tasks (\textsc{hard} demos).
   Each cell reports mean normalized goal proximity ($\uparrow$) $\pm$ standard error over 30 trials.
   All imitation learning baselines (GAIL, AIRL, BC) achieve $\leq 0$ on both tasks with \textsc{hard} demos; GT-RL achieves $0.96\!\pm\!0.01$ / $0.90\!\pm\!0.01$ (reach / pick-place). See Table~\ref{tab:baselines-hard} in the appendix for full baseline results.}
  \label{tab:sri-data-efficiency}
  \begin{subtable}[t]{0.48\linewidth}
    \centering
    \subcaption{Reach}
    \label{tab:reachdata}
    \begin{tabular}{lccc}
      \toprule
      \multirow{2}{*}{\makecell[b]{\# Tasks}}
         & \multicolumn{3}{c}{\# Observations / Task} \\
      \cmidrule(lr){2-4}
         & 100 & 1k & 10k \\
      \midrule
      1280 & $0.86\!\pm\!0.01$ & $0.92\!\pm\!0.00$ & $0.93\!\pm\!0.00$ \\
      320  & $0.78\!\pm\!0.01$ & $0.87\!\pm\!0.01$ & $0.89\!\pm\!0.01$ \\
      80   & $0.50\!\pm\!0.02$ & $0.79\!\pm\!0.01$ & $0.81\!\pm\!0.01$ \\
      \bottomrule
    \end{tabular}
  \end{subtable}
  \hfill
  \begin{subtable}[t]{0.48\linewidth}
    \centering
    \subcaption{Pick-place}
    \label{tab:ppdata}
    \begin{tabular}{lccc}
      \toprule
      \multirow{2}{*}{\makecell[b]{\# Tasks}}
         & \multicolumn{3}{c}{\# Observations / Task} \\
      \cmidrule(lr){2-4}
         & 100 & 1k & 10k \\
      \midrule
      1280 & $0.61\!\pm\!0.08$ & $0.72\!\pm\!0.07$ & $0.78\!\pm\!0.05$ \\
      320  & $0.22\!\pm\!0.07$ & $0.77\!\pm\!0.05$ & $0.71\!\pm\!0.06$ \\
      80   & $0.02\!\pm\!0.02$ & $0.27\!\pm\!0.07$ & $0.50\!\pm\!0.09$ \\
      \bottomrule
    \end{tabular}
  \end{subtable}
\end{table}



\textbf{Failure modes.} SRI's failures were largely data failures: when the state-reward dataset did not cover all regions that a policy might explore, SRI sometimes inferred hackable reward functions with incorrect local optima. Pick-place tasks were particularly sensitive, as fitting precise shaping rewards alongside a large, discontinuous success reward strained both data and model capacity (Table \ref{tab:ppdata}). These failures are also hard to regularize away: unlike in imitation learning, the inferred policy and reward cannot be regularized toward the demonstrations themselves, though regularizing toward states covered by the state-reward dataset or collecting state-reward labels iteratively on RL-visited states may be a viable substitute.

\subsection{Shah et al.\ gridworld benchmark}
\label{sec:shah_gridworld_results}

To compare SRI against the closest prior method for reward inference from arbitrary suboptimal behavior, we also evaluate on the tabular gridworld benchmark of \citet{shah2019feasibility}; see Figure~\ref{fig:shah_gridworld}.
This benchmark differs substantially from our Meta-World experiments.
Most importantly, it exposes a full demonstrator policy on a tabular domain, whereas our main SRI setting conditions on a set of behavior trajectories and predicts rewards at queried states. We therefore use a policy-input variant of SRI that preserves the same supervised inverse-mapping perspective while matching the benchmark interface.
Rather than first learning an explicit differentiable planner and then inverting it by gradient descent \citep{shah2019feasibility}, this variant maps the observed policy directly to the underlying reward map (see
Appendix~\ref{app:shah_gridworld} for details).

SRI performs well in this setting.
As shown in Figure~\ref{fig:shah_gridworld}, it achieves near-ceiling performance with every class of suboptimal demonstrator. We conjecture that SRI's advantage comes from avoiding the planner-learning bottleneck:
their Algorithm~1 must approximate the demonstrator with a value iteration network and then solve a separate inverse problem through that learned planner, while SRI instead uses the same supervision to learn the inverse map directly. Although this abandons any RL-specific inductive biases, it is possible that even $5000$ policy--reward pairs may be enough to favor tabula rasa machine learning methods \citep{sutton2019bitter}. Taken together with our Meta-World results, these results show that SRI can infer rewards from arbitrary fixed behavior classes as accurately as the behavior allows, both in continuous-control settings beyond the scope of prior methods and on the tabular benchmark where those methods were originally introduced.

\begin{figure}[t]
\centering
\includegraphics[width=\linewidth]{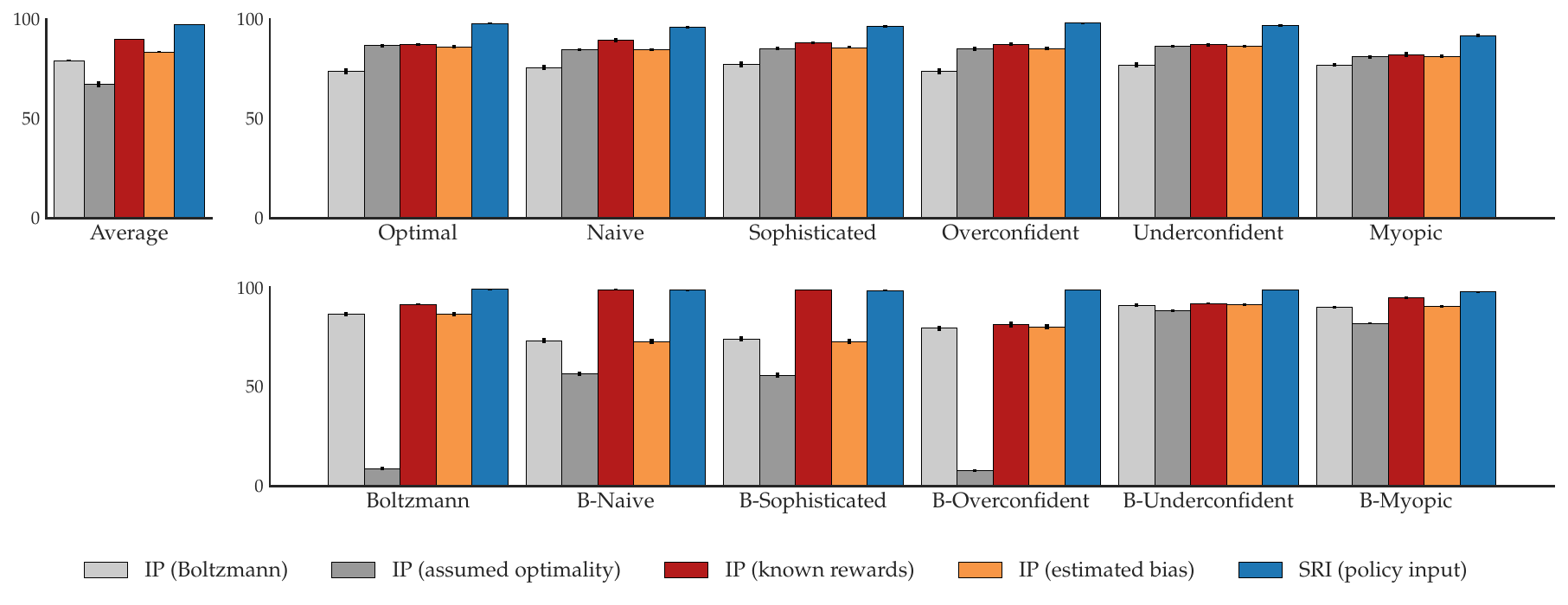}
\caption{\textbf{Comparison on the Shah et al.~(2019) gridworld benchmark.}
Mean percent reward recovered ($\uparrow$) over 30 trials for six deterministic demonstrator classes and their Boltzmann-noisy counterparts, using the evaluation protocol of \citet{shah2019feasibility}: each bar reports the discounted true return obtained by planning optimally with the inferred reward, normalized by the return obtained when planning with the ground-truth reward.
The four inverse-planning (IP) baselines are from \citet{shah2019feasibility}: IP~(Boltzmann) and IP~(assumed optimality) assume a fixed behavior model; IP~(known rewards) is their Algorithm~1, which learns a differentiable planner from reward--policy pairs; and IP~(estimated bias) jointly estimates the planner and a bias term.
SRI~(policy input) is our method, adapted to condition on a full demonstrator policy rather than a set of trajectories and to regress directly to the tabular reward map. On this set of seeds,
SRI attains the strongest overall average and matches or exceeds Algorithm~1 in every condition, suggesting that directly learning the policy-to-reward inverse map can outperform planner-learning-plus-inversion even in the tabular setting for which the latter was originally proposed.
Error bars show $\pm$ standard error.}
\label{fig:shah_gridworld}
\end{figure}

\section{Conclusion}
\label{sec:conclusion}
In this paper, we introduced supervised reward inference (SRI), a framework that reduces reward inference from arbitrary (but fixed-distribution) behavior to supervised learning on known rewards. We showed that SRI is asymptotically Bayes-optimal, as long as the model is strong enough: as the quantity of training data available to it approaches infinity, its learned policies approach the highest ground-truth return possible given limited behavior at inference time. Finally, we showed that SRI infers difficult reward functions from behavior with complex suboptimalities in continuous-control tasks \citep{yu2019metaworld} as accurately as that behavior allows, while also performing strongly on the tabular benchmark where prior arbitrary-behavior reward-inference methods were originally evaluated \citep{shah2019feasibility}.


Several directions remain for future work. Our experiments assume task-indexed reward supervision from known simulator rewards; extending SRI to settings where rewards must be estimated -- from sparse annotations \citep{Knox2009tamer}, preferences \citep{christiano2017deep}, facial expressions \citep{cui2020empathic}, or other signals -- is an important open problem. Visual domains will require an image encoder, and SRI's data efficiency may benefit from self-supervised pretraining with visual \citep{oquab2024dino} or behavioral \citep{pirotta2024fast} foundation models. Orthogonally, RLZero \citep{sikchi2025rlzero} highlights a complementary scaling path: SRI treats arbitrary behavior as supervision for reward inference, whereas RLZero bypasses rewards and tackles direct policy inference from language without in-domain supervision.

More broadly, SRI reframes reward inference from arbitrary behavior as a data-and-supervision problem rather than a demonstrator-modeling problem. That shift may allow future methods to leverage the same scaling trends that have driven progress elsewhere in machine learning \citep{brown2020language,oquab2024dino}.

\section*{Acknowledgments}
We are grateful to Bruno Castro da Silva for his exceptionally generous help in pushing this project to completion. We are also grateful to Russell Coleman for his early work on this project.

This work has taken place in the Safe, Correct, and Aligned Learning and Robotics Lab (SCALAR) at The University of Massachusetts Amherst. SCALAR research is supported in part by the NSF (IIS-2437426) and Open Philanthropy. This work was also supported by the NSF under Grant No. CCF-2018372.

Scott Niekum holds concurrent appointments as an Associate Professor at the University of Massachusetts Amherst and as an Amazon Scholar. This paper describes work performed at the University of Massachusetts Amherst and is not associated with Amazon.

\bibliography{mybib}
\bibliographystyle{rlj}

\beginSupplementaryMaterials
\section{Proof of the Bayes Optimality of SRI}
\label{app:theory}

In this section, we use the Bayesian inverse reinforcement learning framework \citep{Ramachandran2007bayesian} to show that any ``ideal'' (empirical-risk-minimizing) SRI algorithm is asymptotically optimal for both reward inference and imitation learning (see the Main Theorem below for details). In particular, our analysis relies on only three main assumptions for the SRI algorithm and the problem: niceness and compactness of the function class and MDP, appropriate model capacity, and data coverage.

Our proof proceeds as follows: first, in Section \ref{sec:prelim}, we lay out our notation and assumptions; second, in Section \ref{sec:rbar-optim}, we derive the closed form of the Bayes-optimal reward function for imitation given limited behavior trajectories; third, in Section \ref{sec:sup-convergence}, we show that SRI converges uniformly in the limit of infinite data to this Bayes-optimal reward function; finally, in Section \ref{sec:policy-convergence}, we use these results to prove that the optimal policies for SRI also converge uniformly to Bayes optimality, i.e., they asymptotically provide the maximum possible expected return given irreducible uncertainty about the ground truth reward function.

\subsection{Preliminaries}
\label{sec:prelim}
\textbf{Notation:} Each of the sets we consider generally has at most one $\sigma$-algebra associated with it. Thus, as standard in probability theory, we will often use each set interchangeably with its measurable space.

Let $\mathcal{R}$ be a measurable space of reward functions in $\mathcal{M}$, equipped with any $\sigma$-algebra, and let $\mathcal{P}(\mathcal{R})$ be the set of all probability measures with respect to $\mathcal{R}$'s $\sigma$-algebra. (Define the operator $\mathcal{P}$ similarly for any measurable space.) Let $P \in \mathcal{P}(\mathcal{R})$ be the ground-truth marginal distribution of reward functions; in particular, these reward functions are potentially co-dependent with behavior trajectories of maximum length $L_B$. 
For any number of trajectories $N$ define the task-generating distribution 
\[\mathcal{D}_R \in \mathcal{P}\left(\mathcal{R} \times \mathcal{T}^N\right).\]

\paragraph{Effective trajectory-input space.}
Let \(T \coloneqq \{\tau_n\}_{n=1}^N\) denote the random trajectory-set input, where \((R,T)\sim\mathcal D_R\), and let \(P_T\) be the marginal distribution of \(T\) under \(\mathcal D_R\). Define
\[
\mathcal X \coloneqq \supp(P_T)\subseteq \mathcal T^N.
\]
Since \(\mathcal T^N\) is compact by Assumption \ref{ass:well-behaved}(a), \(\mathcal X\) is closed and thus compact. All uniform (supremum) statements over trajectory inputs below are taken over \(\mathcal X\).

Rather than seeing the actual reward function, an SRI algorithm sees a set of $M$ state-reward pairs. In particular, let $P_S \in \mathcal{P}(\mathcal{S})$ be a distribution over states. The data-generating distribution is defined by the following process: first, for dataset size (number of tasks) $K$, take $K$ i.i.d. samples 
\[\left\{\left(R_k, \{\tau_{k, n}\}_{n=1}^N\right)\right\}_{k=1}^K\] 
from $\mathcal{D}_R$, then sample $M$ states from $P_S$ for each task: \[\left\{s_{k, m}\right\}_{k=1, m=1}^{K, M}.\] 
Finally, label each state $s_{k, m}$ with its reward for SRI to predict, $R_k(s_{k, m})$. The resulting dataset has $K$ tasks, and each task has $N$ trajectories and $M$ state-reward pairs: 
\[\left\{\left(\left\{\left(s_{k, m}, R_k(s_{k,m})\right)\right\}_{m=1}^M, \left\{\tau_{k,n}\right\}_{n=1}^N\right)\right\}_{k=1}^K.\] 

We aim to learn a parameterized reward function $R_\theta: \mathcal{S} \times \mathcal{T}^N \rightarrow \mathbb{R}$ on this dataset, for $\theta$ in some space $\Theta$ (see Assumption \ref{ass:well-behaved}).

\paragraph{Notation} For brevity, we henceforth omit index specifications when they are clear from context; e.g., $\{\tau_n\} \vcentcolon = \{\tau_n\}_{n=1}^N$, and similar for sequences. Furthermore, in conditional statements we omit the conditioned random variable when it is clear from context: $\mathbb{E}[Y | x] \vcentcolon = \mathbb{E}[Y | X=x]$.

We first lay out all assumptions necessary for our results.

\begin{ass}[Well-Behaved Spaces and Functions]
\label{ass:well-behaved}
\leavevmode
    \textbf{(a)} The state space \( \mathcal{S} \), observation space \( \mathcal{O}\), and action space \( \mathcal{A} \) are compact measurable metric spaces equipped with $\sigma$-algebras \( \mathcal{F}_\mathcal{S} \), \(\mathcal{F}_\mathcal{O}\) and \( \mathcal{F}_\mathcal{A} \), respectively, and all relevant probability measures (e.g., $\pi(s)$ for any $s \in \mathcal{S}$) are defined with respect to $\mathcal{F}_\mathcal{S}$, \(\mathcal{F}_\mathcal{O}\), and $\mathcal{F}_\mathcal{A}$. (It follows that $\mathcal{T}^N$ is also a compact metric space, using some reasonable product metric.)
    \textbf{(b)} The transition probability function $p: \mathcal{S} \times \mathcal{A} \rightarrow \mathcal{P}(\mathcal{S})$ is a Markov kernel (hence measurable), all policies $\pi: \mathcal{S} \rightarrow \mathcal{P}(\mathcal{A})$ are Markov kernels, and the reward function $R: \mathcal{S} \rightarrow \mathbb{R}$ is measurable. 
    \textbf{(c)} The space $\Theta$ of reward model parameters is compact.
    \textbf{(d)} $R_\theta(s, \{\tau_n\})$ is measurable and continuous with respect to $\theta$, $s$, and $\{\tau_n\}$, and $\left\{R_\theta\right\}_{\theta \in \Theta}$ is equicontinuous in $s$ and $\{\tau_n\}$. 
    \textbf{(e)} All measure spaces are $\sigma$-finite.
\end{ass}

\begin{remark}
\label{rem:boundedtheta}
    Compactness of $\Theta$ and $\mathcal{S}$ and equicontinuity (Assumption 1(a, c, d)) imply that $\{R_\theta\}_{\theta \in \Theta}$ is uniformly bounded.
\end{remark}

\begin{ass}[Bounded Returns]
\label{ass:bounded}
    All reward functions $R \in \mathcal{R}$ are bounded, and either $\gamma < 1$ or $\exists L \in \mathbb{Z}_{\geq 0}$ such that for all $t > L$, $R_t = 0$. 
    Therefore, all returns are bounded, and $J_\theta$ is bounded for all $\theta$.
\end{ass}

\begin{ass}[MSE]
\label{ass:mse}
    In this section, SRI is defined using squared error over the dataset. In particular, we define the single-sample loss for a parameter \(\theta\) as
\[
\ell_\theta(R,\{\tau_n\},s) = (R_\theta(s,\{\tau_n\}) - R(s))^2.
\]
\end{ass}

\begin{remark}
\label{rem:nouniquetheta}
    We do not assume the uniqueness of minimizing parameters, but our definitions naturally imply unique minimizing functions.
\end{remark}

\begin{ass}[Model Capacity]
\label{ass:modelcapacity}
    The hypothesis class \(\{R_\theta\}\) contains a continuous version of the posterior-mean regression target on \(\mathcal S\times\mathcal X\). That is, there exists \(\theta^\dagger\in\Theta\) such that, defining
    \[
    \bar{R}_|(s,t)\coloneqq R_{\theta^\dagger}(s,t),\qquad (s,t)\in\mathcal S\times\mathcal X,
    \]
    we have
    \[
    R_{\theta^\dagger}(S,T)=\mathbb E[R(S)\mid T]\quad\text{a.s. under }P_S\times P_T.
    \]
\end{ass}

\begin{ass}[Data Coverage]
\label{ass:datacoverage}
    The data distribution over states has full support on $\mathcal{S}$: $\supp(P_S) = \mathcal{S}.$ (If the MDP contains unreachable states for any reason, the support need not include those states.)
\end{ass}

\subsection{Optimality of $\bar{R}$ as a reward function for imitation}
\label{sec:rbar-optim}

We first adapt a tabular result by \citet{Ramachandran2007bayesian} to show that, when optimizing expected return under an uncertain reward function, it is always optimal to optimize return under the expected value of that reward function.

\begin{lemma}
\label{lem:linearity}
 Let \( M_R = (\mathcal{S}, \mathcal{A}, p, R, d_0, \gamma) \) be an MDP with random reward function $R$ following any distribution \( P_R \in \mathcal{P}(\mathcal{R})\). Define the expected reward function \( \bar{R}: \mathcal{S} \rightarrow \mathbb{R} \) to be \( \bar{R}(s) = \mathbb{E}[R(s)] \). Then, the policy \( \pi^* \) that maximizes the expected cumulative reward \( \mathbb{E}_R[J_R(\pi)] \), where
\[
J_R(\pi) = \mathbb{E}\left[ \sum_{t=0}^\infty \gamma^t R(S_t) \,; \pi \right],
\]
is an optimal policy for the MDP \( M_{\bar{R}} = (\mathcal{S}, \mathcal{A}, p, \bar{R}, d_0, \gamma) \) with reward function \( \bar{R} \). Similarly, any optimal policy for $M_{\bar{R}}$ is also optimal under $J_R$.
\end{lemma}

\begin{proof}
We aim to show that
\begin{equation}\label{eq:main-equality}
\mathbb{E}\left[ J_R(\pi) \right] = \mathbb{E}\left[ \sum_{t=0}^\infty \gamma^t \bar{R}(S_t) \,; \pi \right],
\end{equation}
where the left expectation is taken over both the randomness in \( R \) and the stochastic transitions in the MDP under policy \( \pi \), while the right expectation is taken over the latter alone.

The proof establishes the following equalities:
\begin{align}
\mathbb{E}_R\left[ J_R(\pi) \right] &= \mathbb{E}_R \left[ \mathbb{E}\left[ \sum_{t=0}^\infty \gamma^t R(S_t) \,; \pi \right] \right] \label{eq1}  \\
&= \mathbb{E}\left[ \mathbb{E}_R\left[ \sum_{t=0}^\infty \gamma^t R(S_t) \right] \,; \pi \right] \label{eq2} \\
&= \mathbb{E}\left[ \sum_{t=0}^\infty \gamma^t \mathbb{E}_R\left[ R(S_t) \right] \,; \pi \right]. \label{eq3}
\end{align}
Step $\eqref{eq2} \rightarrow \eqref{eq3}$ follows immediately from Assumption \ref{ass:bounded} and the Bounded Convergence Theorem. Therefore, all that remains is to justify the application of Fubini's Theorem in step $\eqref{eq1} \rightarrow \eqref{eq2}$.

Let \( (\Omega, \mathcal{F}, \mu) \) be the product measure space of reward functions and trajectories, where \( \Omega = \mathcal{R} \times \mathcal{S}^\infty \), \( \mathcal{F} = \mathcal{R} \otimes \mathcal{S}^\infty \) (using the usual cylinder $\sigma$-algebra), and \( \mu = P(R) \times \mathbb{P}_\pi \). Here, \( \mathbb{P}_\pi \) is the probability measure over trajectories induced by \( \pi \) in this MDP.

Define the function \( f(R, \{ S_t \}) = \sum_{t=0}^\infty \gamma^t R(S_t) \).

To apply Fubini's Theorem, we need to verify that
\( f \) is measurable with respect to \( \mathcal{F} \) 
and that 
\( \int_\Omega |f| \, d\mu < \infty \).

\textbf{Measurability:} First, note that the projection of $\Omega$ onto each individual $S_t$ is measurable by definition of the cylinder $\sigma$-algebra; thus, because $R$ is measurable by assumption, each summand of $f$ is individually measurable. The function \(f\) is therefore the limit of finite sums of measurable functions, and hence is measurable.

\textbf{Integrability:} For some non-negative $c \in \mathbb{R}$, whose exact value depends on $\gamma$ and $L$ (if applicable) we have
\begin{align*}
\int_\Omega |f| \, d\mu &= \int_{\mathcal{R}} \int_{S^\infty} \left| \sum_{t=0}^\infty \gamma^t R(S_t) \right| \, d\mathbb{P}_\pi(\{ S_t \}) \, dP(R) \\
&= \mathbb{E}_{R, S_t}\left[ \left| \sum_{t=0}^\infty \gamma^t R(S_t) \right| \right] \\
&\leq \mathbb{E}_{R, S_t}\left[ c \right],
\end{align*}

by Assumption \ref{ass:bounded}.

Thus, we have established \eqref{eq:main-equality}. To conclude, since the expected cumulative reward under the distribution \( P(R) \) equals the expected cumulative reward in the MDP with reward function \( \bar{R} \), maximizing \( \mathbb{E}[J_R(\pi)] \) over policies \( \pi \) is equivalent to maximizing \( J_{\bar{R}}(\pi) \) in the MDP \( M' = (S, A, p, \gamma, \bar{R}) \).

Therefore, the policy \( \pi^* \) that maximizes \( \mathbb{E}[J_R(\pi)] \) is the optimal policy for the MDP with reward function \( \bar{R} \).
\end{proof}

Equipped with the result of Lemma \ref{lem:linearity}, we now show that, under our assumptions, any SRI algorithm asymptotically produces optimal policies for $\bar{R}$, and thus optimal policies for the posterior distribution over rewards, $P_{R|\{\tau_n\}}$. 

\subsection{Convergence of $(R_{\theta_K})$}
\label{sec:sup-convergence}

The first step is to demonstrate the convergence of the learned SRI model in the limit of infinite data. First, we define our risk functions in a manner that allows us to invoke a Glivenko-Cantelli argument \citep{Sen2022gentle}. As discussed in the preliminaries, consider a single sample \((R,\{\tau_n\}, s)\) drawn according to \(\mathcal{D}_R \times P_S\), where \((R,\{\tau_n\}) \sim \mathcal{D}_R\) and \(s \sim P_S\). 
The population risk is then
\[
L(\theta) = \mathbb{E}_{(R,\{\tau_n\}) \sim \mathcal{D}_R, s \sim P_S}[\ell_\theta(R,\{\tau_n\},s)].
\]

\begin{remark}
\label{remark:leastsquaresmin}
    As always for least squares problems, for any \(\theta^* \in \argmin_{\theta \in \Theta} L(\theta)\), we have
    \[
    R_{\theta^*}(S,T)=\bar{R}_|(S,T)\quad\text{a.s. under }P_S\times P_T.
    \]
\end{remark}

\begin{lemma}[Continuous a.s.\ equality implies equality on the support]
\label{lem:ae-to-support}
Let \(f,g:\mathcal S\times\mathcal X\to\mathbb R\) be continuous. If
\[
f(S,T)=g(S,T)\quad\text{a.s. under }P_S\times P_T,
\]
with \(\supp(P_S)=\mathcal S\) (Assumption~\ref{ass:datacoverage}) and \(\supp(P_T)=\mathcal X\), then
\[
f(s,t)=g(s,t)\quad\text{for all }(s,t)\in\mathcal S\times\mathcal X.
\]
\end{lemma}

\begin{proof}
Let \(h=f-g\). Then \(h\) is continuous and \(h(S,T)=0\) almost surely. Suppose for contradiction that there exists \((s_0,t_0)\in\mathcal S\times\mathcal X\) with \(h(s_0,t_0)\neq 0\). By continuity, there is an open neighborhood \(U\) of \((s_0,t_0)\) and \(\varepsilon>0\) such that \(|h|>\varepsilon\) on \(U\). But \((s_0,t_0)\in \supp(P_S\times P_T)\), so \((P_S\times P_T)(U)>0\), contradicting \(h(S,T)=0\) almost surely.
\end{proof}

Remark \ref{remark:leastsquaresmin} and Lemma \ref{lem:linearity} show that the optimal reward function for MSE SRI indeed produces Bayesian-optimal policies. We must now show that this happens in the limit of infinite data, as well. Specifically, we show: 1) that SRI algorithms indeed converge to some optimal $\theta^*$; 2) that this convergence is uniform on \(\mathcal S\times\mathcal X\); 3) that uniform convergence implies convergence of the optimal policies. Steps 1 and 2 follow from standard learning theory patterns in Lemmas \ref{lem:pointwise} and \ref{lem:uniform}. Step 3 is completed in the Main Theorem.

Given a dataset of \(K\) tasks and \(M\) state samples per task,
\[
\{(R_k,\{\tau_{n,k}\})\}_{k=1}^K \sim \mathcal{D}_R^K, \quad \{s_{m,k}\}_{m=1}^M \sim P_S^M,
\]
we form the empirical risk:
\[
\hat{L}_K(\theta) = \frac{1}{K M} \sum_{k=1}^K \sum_{m=1}^M \ell_\theta(R_k,\{\tau_{n,k}\},s_{m,k}).
\]

\begin{remark}
\label{rem:boundedrisk}
    Under Assumptions \ref{ass:well-behaved}, \ref{ass:bounded}, and \ref{ass:mse}, $l_\theta$, $L(\theta)$ and $\hat{L}_K(\theta)$ are all bounded, continuous and measurable. (See Remark \ref{rem:boundedtheta}.)
\end{remark}

\begin{remark}
\label{rem:onlyk}
    Because increasing $K$ provides additional data of all types, we can ignore $M$ in the empirical risk ($M$ need not approach $\infty$). In particular, note that $L(\theta)$ remains the same even if $l_\theta$ is a sample risk over states $s_m$ (i.e., a loss for all state-reward samples in a single task) instead of a loss for one individual sample.
\end{remark}

We can now show both pointwise and uniform convergence of $R_{\theta_K}$.

\begin{lemma}[Pointwise Convergence of $R_{\theta_K}$]
\label{lem:pointwise}
Suppose \(\theta_K \in \arg\min_{\theta \in \Theta}\hat{L}_K(\theta)\) for each \(K\). Under Assumptions \ref{ass:well-behaved}$-$\ref{ass:datacoverage}, even without a unique minimizer \(\theta^*\), we have for all \((s,t)\in\mathcal S\times\mathcal X\):
\[
R_{\theta_K}(s,t) \xrightarrow{\text{a.s.}} \bar{R}_|(s,t).
\]
\end{lemma}

\begin{proof}
As usual for this type of result, our proof follows three steps: \textbf{1)} we show that $l_\theta$ is Glivenko-Cantelli; \textbf{2)} we conclude that $\hat{L}_K$ almost surely converges uniformly to $L$; \textbf{3)} we use the existence of a convergent subsequence $\theta_{K_j}$ to conclude the desired result.

\paragraph{Steps 1 and 2} By Remark \ref{rem:boundedrisk}, we can directly conclude that $l_\theta$ is Glivenko-Cantelli \citep[see, for example, Remark 3.1 and Theorem 3.2,][]{Sen2022gentle}. Thus, by definition, $\hat{L}_K$ almost surely uniformly converges to $L$:
\[
\sup_{\theta \in \Theta}|\hat{L}_K(\theta) - L(\theta)| \xrightarrow{\text{a.s.}} 0.
\]

\paragraph{Step 3} On the almost-sure event where \(\hat{L}_K\) uniformly converges to \(L\), because \(\Theta\) is compact, it is nearly sufficient to show that any convergent subsequence of \(\theta_K\) converges to a minimizer of \(L(\theta)\).

Since $\Theta$ is compact, let $(\theta_{K_j})$ be a convergent subsequence of $\theta_{K}$, and let $\theta^*$ be its limit. Fix $\varepsilon > 0$, and set $K_0$ such that for all $K \geq K_0$, $\sup_{\theta \in \Theta}|\hat{L}_K(\theta) - L(\theta)| < \frac{\varepsilon}{2}$. Then choose $j_0$ such that for all $j \geq j_0$, both $K_j \geq K_0$ and $|L(\theta_{K_j}) - L(\theta^*)| < \frac{\varepsilon}{2}$ (recall that $L$ is continuous). Then, for all $j \ge j_0$,
\begin{align*}
    |\hat{L}_{K_j}(\theta_{K_j}) - L(\theta^*)| \leq& |\hat{L}_{K_j}(\theta_{K_j}) - L(\theta_{K_j})| \\ &+ |L(\theta_{K_j}) - L(\theta^*)| \\
    <& \frac{\varepsilon}{2} + \frac{\varepsilon}{2} \\
    =& \varepsilon,
\end{align*}
so $\hat{L}_{K_j}(\theta_{K_j}) \to L(\theta^*)$. But by definition, for any $\theta \in \Theta$ and all $j$,
\begin{align*}
    & \hat{L}_{K_j}(\theta_{K_j}) \leq \hat{L}_{K_j}(\theta)\\
    &\Rightarrow \lim_{j\rightarrow \infty} \hat{L}_{K_j}(\theta_{K_j}) \leq \lim_{j\rightarrow \infty} \hat{L}_{K_j}(\theta) \\
    &\Rightarrow L(\theta^*) \leq L(\theta).
\end{align*}
Hence, \(\theta^* \in \argmin_{\theta' \in \Theta} L(\theta')\). Letting \(\Theta \supseteq \Theta^* = \argmin_{\theta \in \Theta} L(\theta)\), if \(d(\theta_K,\Theta^*)\) did not converge to \(0\), then some subsequence would remain at distance at least \(\delta>0\) from \(\Theta^*\); by compactness it would admit a convergent subsubsequence, contradicting the fact that every convergent subsequence limit lies in \(\Theta^*\). Therefore
\[
d(\theta_K, \Theta^*)\footnote{Standard point-set distance: \(d(x, A) = \inf_{a \in A} |x - a|\).} \xrightarrow{\text{a.s.}} 0.
\]
Fix \((s,t)\in\mathcal S\times\mathcal X\), and choose \(\theta_K^*\in\Theta^*\) such that \(d(\theta_K,\theta_K^*)=d(\theta_K,\Theta^*)\), so \(d(\theta_K,\theta_K^*)\to 0\) a.s. By continuity of \(\theta\mapsto R_\theta(s,t)\), we get
\[
R_{\theta_K}(s,t)-R_{\theta_K^*}(s,t)\xrightarrow{\text{a.s.}}0.
\]
By Remark \ref{remark:leastsquaresmin}, \(R_{\theta_K^*}(S,T)=\bar R_{|}(S,T)\) a.s.; by Lemma \ref{lem:ae-to-support}, this implies
\[
R_{\theta_K^*}(s,t)=\bar R_{|}(s,t)\quad\forall (s,t)\in\mathcal S\times\mathcal X.
\]
Therefore \(R_{\theta_K}(s,t)\to \bar R_{|}(s,t)\) almost surely.
\end{proof}

\begin{lemma}[Uniform Convergence of $(R_{\theta_K})$]
\label{lem:uniform}

Define $\theta_K$ as in Lemma \ref{lem:pointwise}. Given the equicontinuity of $\{R_\theta\}$, we also have almost sure uniform convergence of $(R_{\theta_K})$ on \(\mathcal S\times\mathcal X\):
\[
\sup_{(s,t)\in\mathcal S\times\mathcal X}\left|R_{\theta_K}(s,t)-\bar{R}_{|}(s,t)\right| \xrightarrow{\text{a.s.}} 0.
\]

\end{lemma}

\begin{proof}

The desired result follows smoothly from pointwise convergence and the Arzel\`a-Ascoli Theorem.

If \((R_{\theta_K})\) did not converge uniformly to \(\bar{R}_{|}\) on \(\mathcal S\times\mathcal X\), then there would exist an \(\varepsilon > 0\), a subsequence \((R_{\theta_{K_i}})\), and a sequence \(\big((s_i,t_i)\big)\in(\mathcal S\times\mathcal X)^\mathbb N\) such that
\[
\left|R_{\theta_{K_i}}(s_i,t_i)-\bar{R}_{|}(s_i,t_i)\right| \geq \varepsilon \text{ for all } i.
\]

Of course, this subsequence also satisfies the conditions of Arzel\`a-Ascoli -- most notably, equicontinuity -- on the compact domain \(\mathcal S\times\mathcal X\), so we can extract a further uniformly convergent subsequence \((R_{\theta_{K_{i_j}}})\). By Lemma \ref{lem:pointwise}, we know that pointwise, \((R_{\theta_{K_{i_j}}})\) almost surely converges to \(\bar{R}_|\) (using the same event for almost sure convergence as all other almost sure convergences here).

Since \(\mathcal S\times\mathcal X\) is compact, \(\big((s_i,t_i)\big)\) has a convergent subsequence; passing to that subsequence (and relabeling), we may assume \((s_{i_j},t_{i_j})\to(s^*,t^*)\). The uniform convergence of \((R_{\theta_{K_{i_j}}})\) together with pointwise convergence to \(\bar{R}_|\) then gives \(R_{\theta_{K_{i_j}}}(s_{i_j},t_{i_j})\to\bar{R}_{|}(s^*,t^*)\), contradicting \(\left|R_{\theta_{K_{i_j}}}(s_{i_j},t_{i_j})-\bar{R}_{|}(s_{i_j},t_{i_j})\right| \geq \varepsilon\). Thus, \((R_{\theta_K})\) must converge uniformly.

\end{proof}

\subsection{Convergence of SRI's optimal policies to Bayesian optimality}
\label{sec:policy-convergence}

Finally, we conclude that any SRI algorithm satisfying our assumptions produces asymptotically Bayes-optimal policies.

\begin{namedthm*}{Main Theorem}
[Asymptotic Optimality of SRI Algorithms]
\label{thm:policy_convergence}
Let $\mathcal{M}_R = (\mathcal{S}, \mathcal{A}, p, R, d_0, \gamma)$ be an MDP with random reward \(R\), which is drawn together with $\{\tau_n\}$ from a distribution $\mathcal{D}_R \in \mathcal{P}(\mathcal{R} \times \mathcal{T}^N)$.

\textbf{Claim:} Any SRI algorithm that globally minimizes its empirical risk and satisfies Assumptions \ref{ass:well-behaved}$-$\ref{ass:datacoverage} is asymptotically optimal in the sense that the policies it produces approach those that maximize the expected return under the posterior distribution of \(R\) given $\{\tau_n\}$.

\end{namedthm*}

\paragraph{Uniform Convergence of Inferred Rewards}
   By Lemma \ref{lem:uniform}, we have that as \(K \to \infty\),
   \[
   \sup_{(s,t)\in\mathcal S\times\mathcal X}|R_{\theta_K}(s,t)-\bar{R}_|(s,t)| \xrightarrow{\text{a.s.}} 0.
   \]
   Here \(\bar{R}_|\) is the Bayes-optimal reward function. More precisely, for \(P_T\)-almost every realized task input \(t\), Lemma \ref{lem:linearity} applied to the conditional distribution \(P_{R\mid T=t}\) shows that the posterior-mean reward \(s \mapsto \mathbb{E}[R(s)\mid T=t]\) is Bayes-optimal. By Assumption \ref{ass:modelcapacity}, \(\bar{R}_|(\cdot,t)\) is a continuous version of this posterior mean. In particular, \(\bar{R}_|\) is the reward function that, if known, would yield the policy maximizing the expected return given \(\{\tau_n\}\).

   Thus, for any \(\delta > 0\), there exists \(K_\delta\) such that for all \(K > K_\delta\),
   \[
   \sup_{(s,t)\in\mathcal S\times\mathcal X}|R_{\theta_K}(s,t)-\bar{R}_|(s,t)| < \delta.
   \]
   Since the test-time trajectory input \(T\) is drawn from \(P_T\), we have \(T\in\mathcal X\) almost surely, so for the realized task input we also have
   \[
   \sup_{s\in\mathcal S}|R_{\theta_K}(s,T)-\bar{R}_|(s,T)|<\delta\quad\text{a.s.}
   \]

\paragraph{Policy Convergence} 
   Let \(\pi_{\theta_K}^*\) be any optimal policy for the MDP \((\mathcal{S}, \mathcal{A}, p, R_{\theta_K}, d_0, \gamma)\). We wish to show that \(\{\pi_{\theta_K}^*\}\) approaches optimality for \(\bar{R}_|\) as \(K \to \infty\).

   Define the value functions under a reward function \(R'\) and policy \(\pi\) as:
   \[
   V_{R'}^\pi(s) = \mathbb{E}\left[\sum_{t=0}^{\infty} \gamma^t R'(S_t,T) \mid S_0=s,\pi\right].
   \]
   Let \(V_{R_{\theta_K}}^*(s) = V_{R_{\theta_K}}^{\pi_{\theta_K}^*}(s)\) be the optimal value under \(R_{\theta_K}\), and let \(V_{\bar{R}_|}^*(s)\) be the optimal value under \(\bar{R}_|\).

   We know:
   \[
   V_{R_{\theta_K}}^{\pi_{\theta_K}^*}(s) \geq V_{R_{\theta_K}}^\pi(s), \quad \forall \pi.
   \]

   We want to relate \(V_{\bar{R}_|}^{\pi_{\theta_K}^*}\) to \(V_{\bar{R}_|}^*\). Consider the difference:
   \[
   V_{\bar{R}_|}^*(s) - V_{\bar{R}_|}^{\pi_{\theta_K}^*}(s).
   \]
   Introduce the intermediate value functions under \(R_{\theta_K}\):
   \begin{align*}
   V_{\bar{R}_|}^*(s) - V_{\bar{R}_|}^{\pi_{\theta_K}^*}(s) &= [V_{\bar{R}_|}^*(s) - V_{R_{\theta_K}}^*(s)] \\ &+ [V_{R_{\theta_K}}^*(s) - V_{\bar{R}_|}^{\pi_{\theta_K}^*}(s)].
   \end{align*}

   We will bound each piece.

   Bounding \(|V_{\bar{R}_|}^*(s) - V_{R_{\theta_K}}^*(s)|\):     
     Because \(\sup_{(s,t)\in\mathcal S\times\mathcal X}|R_{\theta_K}(s,t)-\bar{R}_|(s,t)| < \delta\), we have for any policy \(\pi\):
     \begin{align*}
         |V_{R_{\theta_K}}^\pi(s) - V_{\bar{R}_|}^\pi(s)| &\leq \mathbb{E}\left[\sum_{t=0}^{\infty}\gamma^t |R_{\theta_K}(S_t,T)-\bar{R}_|(S_t,T)|\right].
     \end{align*}
     By Assumption \ref{ass:bounded}, we are in one of two cases:
     \[
     C_{\gamma,L}\coloneqq
     \begin{cases}
       \frac{1}{1-\gamma}, & \gamma<1,\\
       L+1, & \gamma=1 \text{ and } R_t=0 \text{ for all } t>L.
     \end{cases}
     \]
     Therefore,
     \[
     |V_{R_{\theta_K}}^\pi(s) - V_{\bar{R}_|}^\pi(s)| \leq C_{\gamma,L}\,\delta.
     \]

     Thus:
     \[
     \|V_{R_{\theta_K}}^\pi - V_{\bar{R}_|}^\pi\|_\infty \leq C_{\gamma,L}\,\delta, \quad \text{for any } \pi.
     \]
     In particular, this applies to the optimal policies under either reward:
     \[
     |V_{\bar{R}_|}^*(s) - V_{R_{\theta_K}}^*(s)| \leq C_{\gamma,L}\,\delta.
     \]

   Bounding \(|V_{R_{\theta_K}}^*(s) - V_{\bar{R}_|}^{\pi_{\theta_K}^*}(s)|\):    
     By optimality of \(\pi_{\theta_K}^*\) under \(R_{\theta_K}\),
     \[
     V_{R_{\theta_K}}^*(s) = V_{R_{\theta_K}}^{\pi_{\theta_K}^*}(s).
     \]
     Thus:
     \[
     V_{R_{\theta_K}}^*(s) - V_{\bar{R}_|}^{\pi_{\theta_K}^*}(s) = [V_{R_{\theta_K}}^{\pi_{\theta_K}^*}(s) - V_{\bar{R}_|}^{\pi_{\theta_K}^*}(s)].
     \]
     This difference is bounded by the same \(C_{\gamma,L}\delta\) argument as above:
     \[
     |V_{R_{\theta_K}}^{\pi_{\theta_K}^*}(s) - V_{\bar{R}_|}^{\pi_{\theta_K}^*}(s)| \leq C_{\gamma,L}\,\delta.
     \]

   Combining these results, we have:
   \[
   V_{\bar{R}_|}^*(s) - V_{\bar{R}_|}^{\pi_{\theta_K}^*}(s) \leq C_{\gamma,L}\delta + C_{\gamma,L}\delta = 2C_{\gamma,L}\delta.
   \]

   Since \(\delta>0\) was arbitrary and can be made as small as desired by taking \(K\) sufficiently large, for any \(\varepsilon > 0\), choose \(\delta = \frac{\varepsilon}{2C_{\gamma,L}}\). Then for all \(K > K_\delta\),
   \[
   V_{\bar{R}_|}^*(s) - V_{\bar{R}_|}^{\pi_{\theta_K}^*}(s) \leq \varepsilon.
   \]
   Therefore:
   \[
   V_{\bar{R}_|}^{\pi_{\theta_K}^*}(s) \geq V_{\bar{R}_|}^*(s) - \varepsilon, \quad \forall s \in \mathcal{S},
   \] and in particular
   \[
   J_{\bar{R}_|}(\pi^*_{\theta_K}) \geq J^*_{\bar{R}_|} - \varepsilon.
   \]

\paragraph{Conclusion}
   We have shown that as \(K \to \infty\), the SRI algorithm's inferred reward functions \(R_{\theta_K}\) converge uniformly to \(\bar{R}_|\), and that the corresponding optimal policies \(\pi_{\theta_K}^*\) approach optimality under \(\bar{R}_|\). Since \(\bar{R}_|\) is the Bayes-optimal reward given \(\{\tau_n\}\) (Lemma \ref{lem:linearity}), the policies derived from the SRI algorithm asymptotically maximize the expected return under the posterior over \(R\).

\section{Architecture Details}
\label{app:arch}

We encoded trajectories using a 3-head transformer with two 256-dimensional transformer layers and a final 2-layer MLP with hidden dimension 50, outputting a trajectory representation of length 100; processed trajectory encodings into a task encoding with a default set transformer \citep{lee2019set} with hidden dimension 128, outputting a task representation of dimension 256; encoded states using a 2-layer MLP with hidden dimension 256 and encoding dimension 100; and processed task and state representations into a final reward with a 2-layer MLP of hidden dimension 256. All activations were leaky ReLU.

\section{SGI, SAI, and Additional Experiments}
\label{app:additionalresults}

\subsection{Imitation learning baselines on \textsc{hard} demonstrations}

Table~\ref{tab:baselines-hard} reports the full baseline results for the \textsc{hard} demonstration condition summarized in Table~\ref{tab:sri-data-efficiency}.

\begin{table}[h]
  \centering
  \footnotesize
  \caption{Imitation learning baselines on Meta-World with \textsc{hard} demos. Mean normalized goal proximity ($\uparrow$) $\pm$ standard error over 30 trials.}
  \label{tab:baselines-hard}
  \begin{tabular}{lcc}
    \toprule
    Method & Reach & Pick-place \\
    \midrule
    GAIL  & $-1.81\!\pm\!0.37$ & $-0.004\!\pm\!0.002$ \\
    AIRL  & $-1.58\!\pm\!0.30$ & $-0.001\!\pm\!0.000$ \\
    BC    & $-0.93\!\pm\!0.08$ & $-0.002\!\pm\!0.001$ \\
    GT-RL & $\phantom{-}0.957\!\pm\!0.005$ & $\phantom{-}0.903\!\pm\!0.011$ \\
    \bottomrule
  \end{tabular}
\end{table}

\subsection{SGI and SAI}

SRI provides an example of a \textit{supervised intent inference} algorithm where ``intent'' is formalized as a reward function. However, this general intent inference framework is easily extended to other formalizations of intent. In this section, we provide definitions and empirical results for two further examples: \textbf{supervised goal inference (SGI)}, where the intent object being mapped to is some latent representation of the task; and \textbf{supervised action inference (SAI)}, where the intent object is a (presumably optimal) policy. In each case, the central conceptual innovation over prior methods (e.g., goal inference or imitation learning) is in viewing the observed behavior as an \textit{indication} of the demonstrator's intent, rather than an attempted \textit{enactment} of it. See Figure \ref{fig:arch_sgi_sai} for an abstract diagram of our implementations of SGI and SAI.

\begin{wrapfigure}[25]{R}{0.5\textwidth}
\centering
    \includegraphics[width=\linewidth]{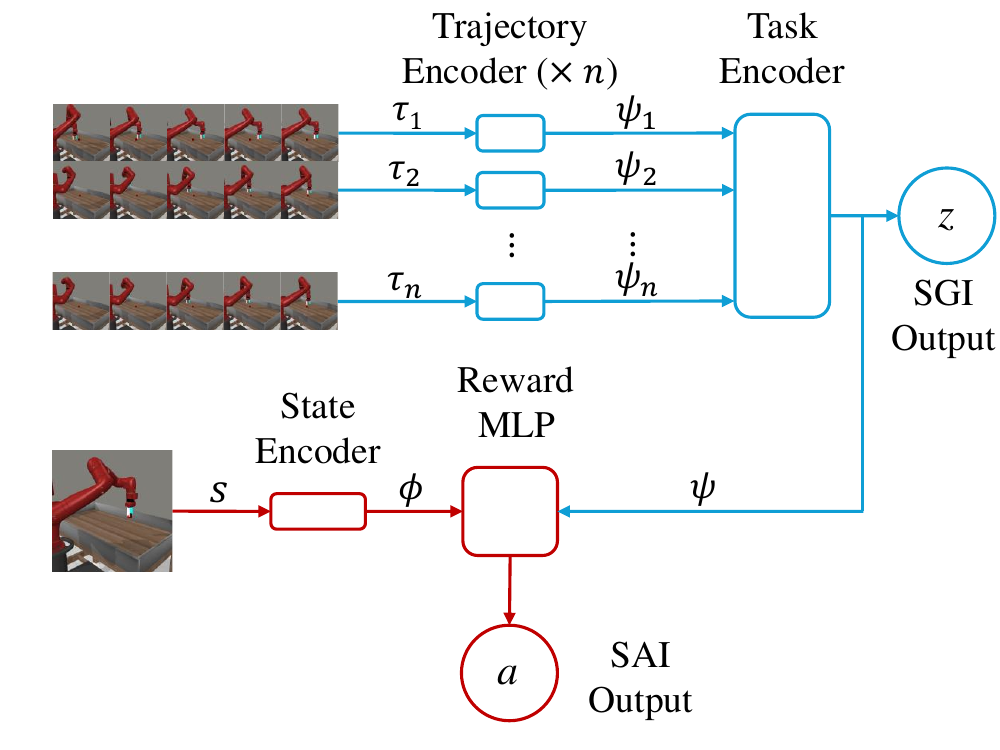}
    \captionof{figure}{%
      \textbf{Architecture for Supervised Goal Inference and Supervised Action Inference.} Trajectories are encoded with a transformer and pooled by a set
      transformer into a task embedding $\psi$ (\textcolor{ppt-blue}{blue}, once per task). In the case of supervised goal inference (SGI), this embedding is directly used as the goal prediction (i.e., $z=\psi$). In the case of supervised action inference (SAI), similar to SRI,
      each state is mapped to an observation embedding $\phi$ by two
      small MLPs (\textcolor{ppt-red}{red}, per timestep); a final MLP combines
      $(\psi,\phi)$ to predict the action.}
    \label{fig:arch_sgi_sai}
\end{wrapfigure}

\subsubsection{SGI}

Supervised goal inference (SGI) takes as its inference target any latent representation of the task that can be used to determine the reward function or an optimal policy. In our experiments with Meta-World's reach and pick-place tasks \citep{yu2019metaworld}, we use the $(x, y, z)$ position of the goal as the target, and then infer actions using Meta-World's oracular policies. However, depending on the structure of the task and available information, inference targets could just as easily consist of the parameters of a neural network, the weights of a linear reward function, or a one-hot task index, among others.

Because the target is not a function (in particular, not a function taking states as inputs), SGI only requires a multi-task dataset pairing behavior trajectories with goal labels.

\begin{definition}[Supervised Goal Inference]

Given: \textbf{a}) a random set of behavior trajectories $\{T_n\}_{n=1}^N \in \mathcal{T}^N$ and a random (uncertain) goal $Z \in \mathbb{R}^{D_z}$ (for some goal dimension $D_z$) jointly following a distribution $\mathcal{D}_T$; \textbf{b}) 
a parameterized function family $f_{\theta}: \mathcal{T}^N \rightarrow \mathbb{R}^{D_z}$; and \textbf{c}) a regression loss function $\mathcal{L}: \mathbb{R}^{D_z} \times \mathbb{R}^{D_z} \rightarrow \mathbb{R}$, \textit{supervised goal inference} (SGI) is the following \textbf{multi-task} minimization problem:
\begin{equation} 
    \argmin_{\theta}\mathbb{E}\bigg[\mathcal{L}\bigg (f_{\theta}(\{T_n\}_{n=1}^N), Z\bigg) \bigg ]. \label{sgi}
\end{equation}
\end{definition}

See Algorithm \ref{alg:sgi} for an example gradient-descent-based implementation with $L^2$ loss.

\subsubsection{SAI}

Supervised action inference (SAI) takes as its inference target any policy. For simplicity of presentation, the following definition assumes that the target policy is deterministic and that the action space is continuous (or at least represented continuously), but allowing stochastic policies and discrete action spaces would require minimal changes. Note that SAI's target actions are not assumed to come from the behavior trajectories; as with SRI, those trajectories are assumed to be observation-only, and potentially arbitrarily suboptimal, thereby serving only as an \textit{indication} of the agent's intended actions: i.e., the input to the intent-inference problem. The intended actions -- the output of the intent-inference problem -- are provided in the form of state-action samples, akin to SRI's state-reward samples.

\begin{definition}[Supervised Action Inference]

Assume that the action space of the MDP is $\mathcal{A} = \mathbb{R}^{D_a}$ for some dimension $D_a$ (i.e., that the actions are continuous or continuously parameterized). Then, given: \textbf{a}) a random set of behavior trajectories $\{T_n\}_{n=1}^N \in \mathcal{T}^N$ and a random (uncertain) policy $\Pi$ jointly following a distribution $\mathcal{D}_T$; \textbf{b}) a random state $S$ following some distribution $\mathcal{D}_S$; \textbf{c}) 
parameterized function families $f_{\theta_f}: \mathcal{T}^N \rightarrow \Psi$ and $g_{\theta_g}: \mathcal{S} \times \Psi \rightarrow \mathbb{R}^{D_a}$; and \textbf{d}) a regression loss function $\mathcal{L}: \mathbb{R}^{D_a} \times \mathbb{R}^{D_a} \rightarrow \mathbb{R}$, \textit{supervised action inference} (SAI) is the following \textbf{multi-task} minimization problem:
\begin{equation} 
    \argmin_{\theta_f, \theta_g}\mathbb{E}\bigg[\mathcal{L}\bigg (g_{\theta_g}\big(S, f_{\theta_f}(\{T_n\}_{n=1}^N) \big), \Pi(S)\bigg) \bigg ]. \label{sai}
\end{equation}
\end{definition}

See Algorithm \ref{alg:sai} for an example gradient-descent-based implementation with $L^2$ loss.

\begin{figure}
  \centering
    \begin{minipage}[t]{\linewidth}
    \captionsetup{type=algorithm}      
    \begin{algorithm}[H]
    \caption{SGI Training with Gradient Descent}
    \label{alg:sgi}
    \begin{algorithmic}[1] 
        \State {\bfseries Input:} Num tasks $K$, num dems/task $N_T$, goal space dimension $D_z$, batch size $M$, num inference dems $N_I$, learning rate $\alpha$, and dataset $\mathcal{D} = \{(\{\tau_{k,i}\}_{i=1}^{N_T}, z_k)\}_{k=1}^K$, where $z_k \in \mathbb{R}^{D_z}$ is the goal for task $k$.
        \State Initialize goal inference model $f_{\theta}: \mathcal{T}^{N_I} \rightarrow \mathbb{R}^{D_z}$ with parameters $\theta$.
        \Repeat
            \State Sample batch $\{(\mathcal{T}_b, z_b)\}_{b=1}^M \sim \mathcal{D}$, where for each batch item $b$, $\mathcal{T}_b=\{\tau_{b,i}\}_{i=1}^{N_T}$ is the set of $N_T$ trajectories and $z_b$ is the corresponding goal.
            \State For each task $b \in [1,M]$, randomly select $N_I$ trajectories $\mathcal{T}'_b \subseteq \mathcal{T}_b$.
            \State Predict goals $(\hat{z}_b)_{b=1}^M = (f_{\theta}(\mathcal{T}'_b))_{b=1}^M$.
            \State Compute loss $\mathcal{L}(\theta) \leftarrow \frac{1}{M}\sum_{b=1}^M \|\hat{z}_b - z_b\|^2_2$.
            \State Update parameters $\theta \leftarrow \theta - \alpha \nabla_{\theta} \mathcal{L}(\theta)$.
        \Until{convergence}
        \State {\bfseries Output:} Learned parameters $\theta$.
    \end{algorithmic}
\end{algorithm}
  \end{minipage}
  
  \begin{minipage}[t]{\linewidth}
    \captionsetup{type=algorithm}      
    \begin{algorithm}[H]
    \caption{SAI Training with Gradient Descent}
    \label{alg:sai}
    \begin{algorithmic}[1] 
        \State {\bfseries Input:} Num tasks $K$, num dems/task $N_T$, num state-action samples/task $N_s$, action space dimension $D_a$, batch size $M$, num inference dems $N_I$, learning rate $\alpha$, and dataset $\mathcal{D} = \{(\{\tau_{k,i}\}_{i=1}^{N_T}, \{(s_{k,j}, a_{k,j})\}_{j=1}^{N_s})\}_{k=1}^K$.
        \State Initialize encoder $f_{\theta_f}: \mathcal{T}^{N_I} \rightarrow \mathbb{R}^d$ and action model $g_{\theta_g}: \mathcal{S} \times \mathbb{R}^d \rightarrow \mathbb{R}^{D_a}$. Let $\theta = (\theta_f, \theta_g)$.
        \Repeat
            \State Sample batch $\{(\mathcal{T}_k, \mathcal{S}_k)\}_{k=1}^M \sim \mathcal{D}$, where $\mathcal{T}_k=\{\tau_{k,i}\}_{i=1}^{N_T}$ are trajectories and $\mathcal{S}_k=\{(s_{k,j}, a_{k,j})\}_{j=1}^{N_s}$ are state-action pairs for task $k$.
            \State For each task $k \in [1,M]$, randomly select $N_I$ trajectories $\mathcal{T}'_k \subseteq \mathcal{T}_k$.
            \State Compute task embeddings $(\psi_k)_{k=1}^M = (f_{\theta_f}(\mathcal{T}'_k))_{k=1}^M$.
            \State Predict actions $(\hat{a}_{k,j})_{k \in [1,M], j \in [1,N_s]} = (g_{\theta_g}(s_{k,j}, \psi_k))_{k \in [1,M], j \in [1,N_s]}$.
            \State Compute loss $\mathcal{L}(\theta) \leftarrow \frac{1}{MN_s}\sum_{k=1}^M \sum_{j=1}^{N_s} \|\hat{a}_{k,j} - a_{k,j}\|^2_2$.
            \State Update parameters $\theta \leftarrow \theta - \alpha \nabla_{\theta} \mathcal{L}(\theta)$.
        \Until{convergence}
        \State {\bfseries Output:} Learned parameters $\theta = (\theta_f, \theta_g)$.
    \end{algorithmic}
\end{algorithm}
  \end{minipage}
\end{figure}

\subsubsection{Experimental performance of SGI and SAI}

\paragraph{Setup} To evaluate SGI and SAI, we reran the six experiments in the main paper using SGI and SAI as drop-in replacements for SRI. Both used the same architecture as SRI, except that SGI's demonstration encodings were 3-dimensional (to match goal size) and SAI's final output was four-dimensional (to match action size). SAI's outputs were not clipped or squeezed during training. The action labels for SAI's prediction task were taken from Meta-World's oracular reach/pick-place policies, with the same minor modifications as described in Appendix \ref{app:demonstrationdetails}. SAI used the same states for labeling (with actions) as SRI (with rewards), though without the pick-place/reach proportion interpolation described in Appendix \ref{app:sritrainingdetails}. Both SGI and SAI were trained for 500 epochs rather than SRI's 2000, as we noted that any resulting loss of precision was insignificant when it could not be hacked by an RL agent.

\paragraph{Results}

\begin{table}[t]
\caption{\textbf{Pick-place (\textsc{gesture} demos).} Mean normalized goal proximity ($\uparrow$) $\pm$ standard error (30 trials).}
\label{tab:app-pickplace_vertical}
\centering
\vspace{0.8em}
\begin{tabular}{lc}
\toprule
\textbf{Method} & \textbf{Prox.\ ($\uparrow$)} \\
\midrule
SRI   & $0.822 \pm 0.051$ \\
SGI   & $0.790 \pm 0.002$ \\
SAI   & $0.791 \pm 0.003$ \\
GAIL  & $-0.001 \pm 0.000$ \\
AIRL  & $-0.001 \pm 0.001$ \\
BC    & $-0.001 \pm 0.000$ \\
GT RL & $0.903 \pm 0.011$ \\
\bottomrule
\end{tabular}
\end{table}

Overall, SGI and SAI performed comparably to SRI, and performed notably \textit{better} than SRI on some reach tasks where the goal is indicated with high certainty by the demonstrations (Tables \ref{tab:app-pickplace_vertical} and \ref{tab:app-reachdata}, Figure \ref{fig:app-offset_reach}, and lower noise levels in Figure \ref{fig:app-noisy_reach}). SGI and SAI also performed better than SRI when given small training datasets, for both reach and pick-place tasks (Table \ref{tab:app-sri-sgi-sai-data-efficiency}). We attribute this latter observation to the fact that slight imperfections in SRI's learned reward may change the optimal policy set dramatically (due to reward hacking), while the degradation of inferred policies with imprecisions in SGI and SAI models is smoother.

\begin{wrapfigure}[21]{R}{0.45\textwidth}
    \centering
    \includegraphics[width=0.44\textwidth]{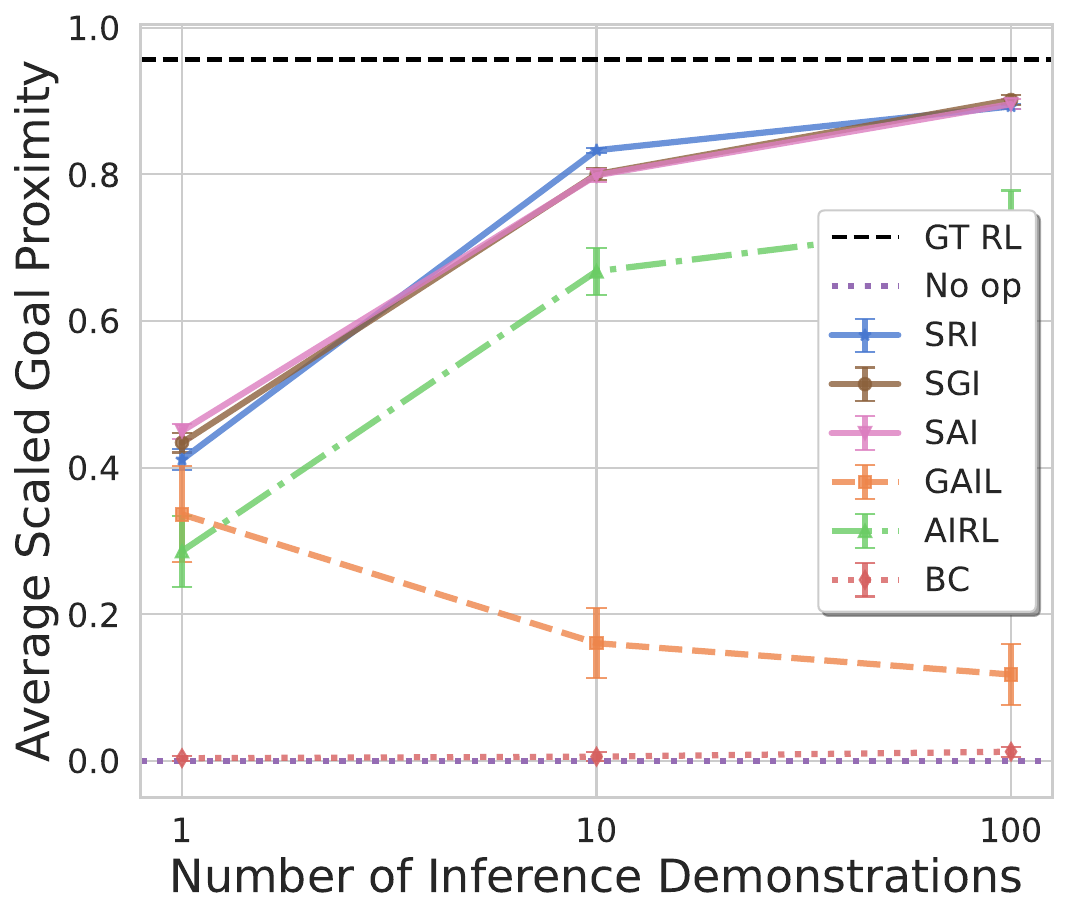}
    \caption{\textbf{Demonstration efficiency.} Mean normalized goal proximity of SRI, SGI, SAI, and imitation baselines on Meta‑World reach tasks using \textsc{noisy}$_{0.87}$ demonstrations, plotted versus the number of demos per task. Error bars show $\pm$ standard error over 30 trials (Sec.~\ref{sec:metrics}).}
    \label{fig:app-reach_dem_efficiency}
\end{wrapfigure}

\paragraph{Discussion} Do these results indicate that one should always use SGI and SAI over SRI? Not necessarily, for two reasons. 

First, when the dynamics of the environment are expected to change, a policy can be retrained with respect to SRI's learned reward, but SAI's inferred policy (and possibly SGI's, if its output is used as the latent variable of a policy, as in our experiments) will become irrevocably incorrect until new data is collected.  

Second, we expect that SRI's requisite dataset will be much easier to collect in many real-world situations, as reward labels may be easier to provide than goal and action labels. In the case of SGI, labeling behavior with a goal (latent representation) requires that the practitioner already know a factorization of the reward function or optimal policy of the underlying task using the given goal, which is not a realistic assumption for many real-world tasks. 

SAI suffers from a similar problem: the practitioner must know the optimal action in any given state, in particular in terms of the agent's action space (e.g., a robot's joint actuations). Not only does this require deep problem-specific expertise even for simple tasks, but it also precludes the agent ever achieving superhuman performance on the target task.

The conclusion of the present results is therefore that a wide variety of intent objects (i.e., labels) have the potential to work well in this supervised intent inference framework, and practitioners should use whichever label type is both easily available and appropriate for their target deployment.

\begin{figure}[t]
  \centering
  \begin{subfigure}{.48\linewidth}
      \includegraphics[width=\linewidth]{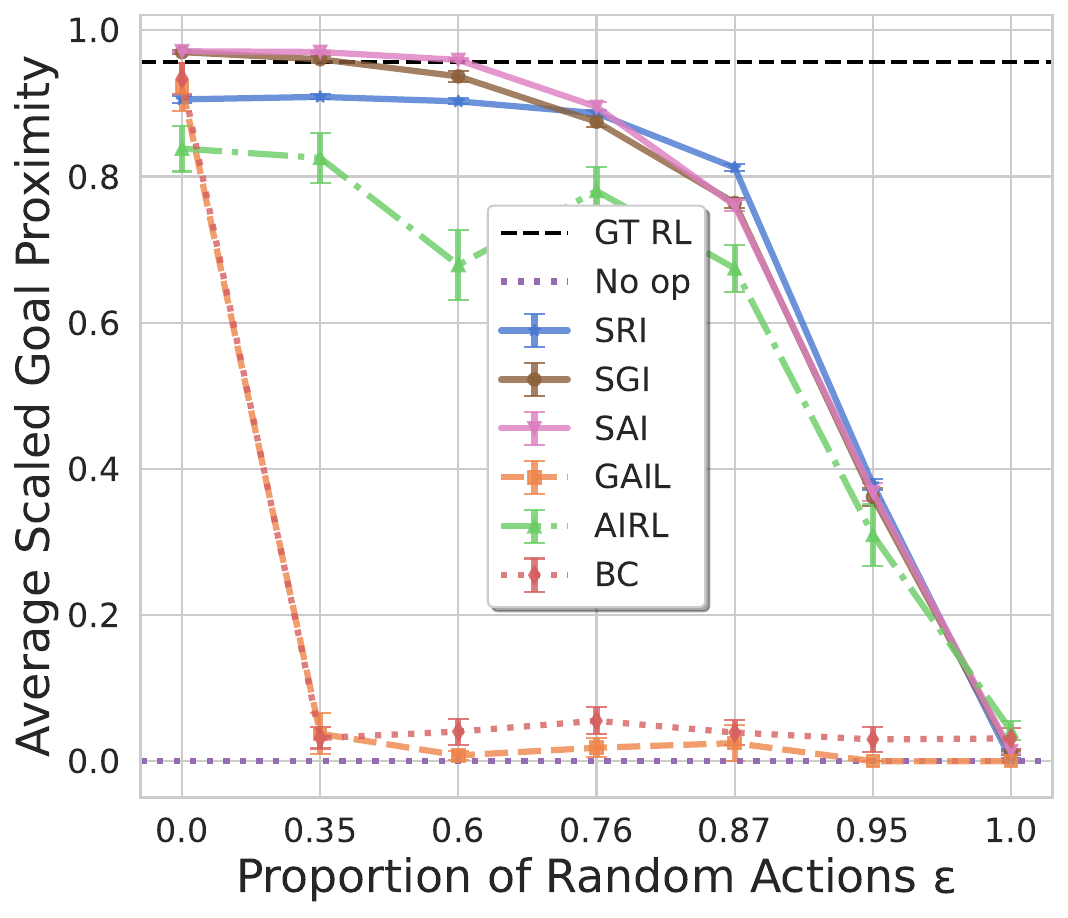}%
      \caption{\textbf{Random actions}: \textsc{noisy}$_\varepsilon$}
      \label{fig:app-noisy_reach}
    \end{subfigure}
  \hfill
  \begin{subfigure}{.48\linewidth}
      \includegraphics[width=\linewidth]{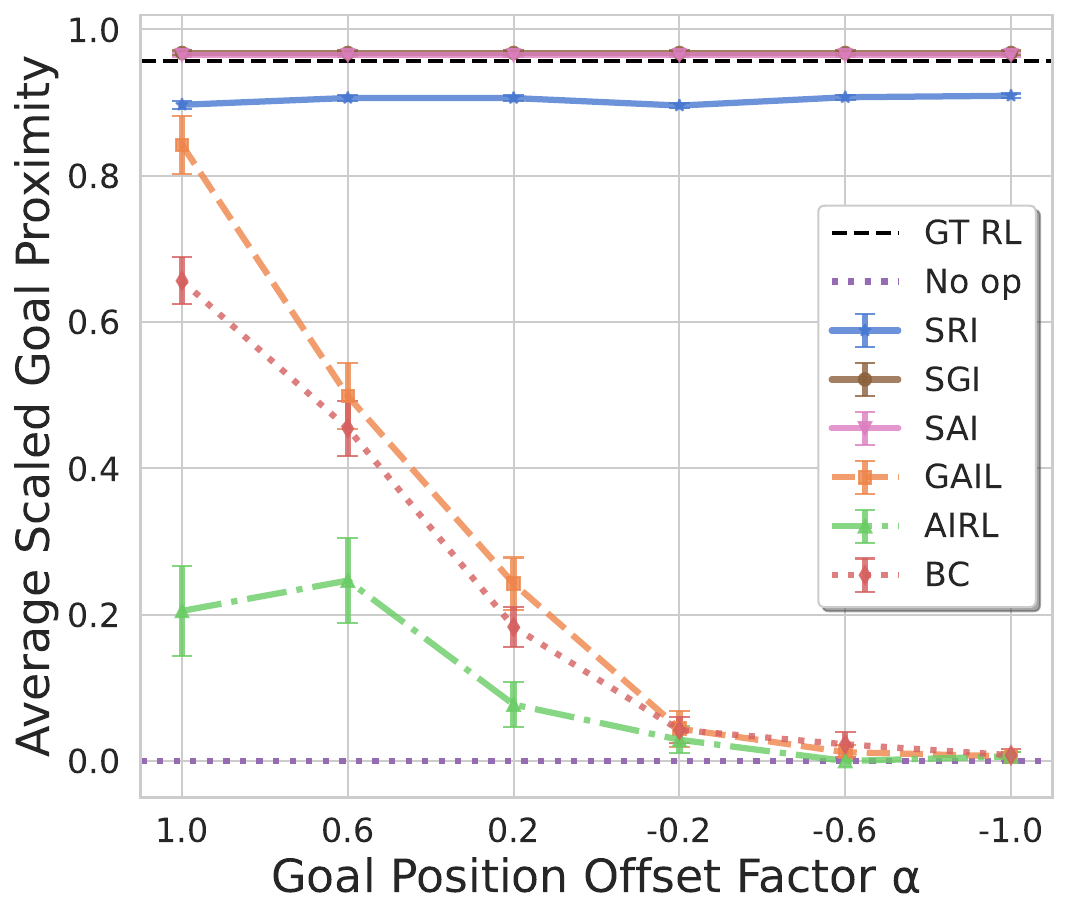}%
      \caption{\textbf{Reach to offset goal}: \textsc{psychic}$_\alpha$}
      \label{fig:app-offset_reach}
      \end{subfigure}
  %
  \caption{\textbf{Robustness of SRI, SGI, and SAI to suboptimal demonstrations in reach tasks.}
           Error bars: standard error over 30 trials. All three supervised methods show near-perfect recovery of optimal policies given \textit{invertible} suboptimalities (\textbf{b}).}
  \label{fig:app-reach_robustness}
\end{figure}

\begin{table}[p] 
  \centering
  \caption{Data-efficiency comparison of \textbf{SRI}, \textbf{SGI}, and \textbf{SAI} on two Meta-World tasks.  Each cell shows mean normalized goal proximity ($\uparrow$) $\pm$ standard error over 30 RL rollout trials, after training on the indicated number of tasks and labeled observations per task.}
  \label{tab:app-sri-sgi-sai-data-efficiency}
  \small
  \vspace{0.6em}

  \begin{subtable}[t]{\linewidth}
    \centering
    \subcaption{Reach task (\textsc{hard} demos)}
    \label{tab:app-reachdata}

    \textbf{SRI}\\[0.2em]
    \begin{tabular}{lccc}
      \toprule
      \multirow{2}{*}{\# Tasks} &
      \multicolumn{3}{c}{\# labeled observations / task} \\
      \cmidrule(lr){2-4}
      & $100$ & $1\,000$ & $10\,000$ \\
      \midrule
      $1\,280$ & $0.860 \pm 0.006$ & $0.916 \pm 0.004$ & $0.930 \pm 0.003$ \\
      $320$    & $0.780 \pm 0.007$ & $0.866 \pm 0.006$ & $0.893 \pm 0.006$ \\
      $80$     & $0.495 \pm 0.017$ & $0.792 \pm 0.008$ & $0.812 \pm 0.010$ \\
      \bottomrule
    \end{tabular}\\[0.8em]

    \textbf{SGI}\\[0.2em]
    \begin{tabular}{lc}
      \toprule
      \# Tasks & Prox.\ ($\uparrow$) \\
      \midrule
      $1\,280$ & $0.972 \pm 0.001$ \\
      $320$    & $0.925 \pm 0.013$ \\
      $80$     & $0.927 \pm 0.010$ \\
      \bottomrule
    \end{tabular}\\[0.8em]

    \textbf{SAI}\\[0.2em]
    \begin{tabular}{lccc}
      \toprule
      \multirow{2}{*}{\# Tasks} &
      \multicolumn{3}{c}{\# labeled observations / task} \\
      \cmidrule(lr){2-4}
      & $100$ & $1\,000$ & $10\,000$ \\
      \midrule
      $1\,280$ & $0.962 \pm 0.003$ & $0.960 \pm 0.004$ & $0.958 \pm 0.004$ \\
      $320$    & $0.924 \pm 0.013$ & $0.929 \pm 0.010$ & $0.923 \pm 0.012$ \\
      $80$     & $0.859 \pm 0.018$ & $0.846 \pm 0.018$ & $0.807 \pm 0.021$ \\
      \bottomrule
    \end{tabular}
  \end{subtable}

  \vspace{1.2em}

  \begin{subtable}[t]{\linewidth}
    \centering
    \subcaption{Pick-place task (\textsc{hard} demos)}
    \label{tab:app-ppdata}

    \textbf{SRI}\\[0.2em]
    \begin{tabular}{lccc}
      \toprule
      \multirow{2}{*}{\# Tasks} &
      \multicolumn{3}{c}{\# labeled observations / task} \\
      \cmidrule(lr){2-4}
      & $100$ & $1\,000$ & $10\,000$ \\
      \midrule
      $1\,280$ & $0.614 \pm 0.076$ & $0.715 \pm 0.069$ & $0.784 \pm 0.052$ \\
      $320$    & $0.221 \pm 0.067$ & $0.768 \pm 0.052$ & $0.708 \pm 0.059$ \\
      $80$     & $0.018 \pm 0.021$ & $0.269 \pm 0.065$ & $0.502 \pm 0.085$ \\
      \bottomrule
    \end{tabular}\\[0.8em]

    \textbf{SGI}\\[0.2em]
    \begin{tabular}{lc}
      \toprule
      \# Tasks & Prox.\ ($\uparrow$) \\
      \midrule
      $1\,280$ & $0.765 \pm 0.005$ \\
      $320$    & $0.757 \pm 0.007$ \\
      $80$     & $0.764 \pm 0.006$ \\
      \bottomrule
    \end{tabular}\\[0.8em]

    \textbf{SAI}\\[0.2em]
    \begin{tabular}{lccc}
      \toprule
      \multirow{2}{*}{\# Tasks} &
      \multicolumn{3}{c}{\# labeled observations / task} \\
      \cmidrule(lr){2-4}
      & $100$ & $1\,000$ & $10\,000$ \\
      \midrule
      $1\,280$ & $0.783 \pm 0.004$ & $0.795 \pm 0.002$ & $0.795 \pm 0.003$ \\
      $320$    & $0.758 \pm 0.005$ & $0.773 \pm 0.006$ & $0.783 \pm 0.005$ \\
      $80$     & $0.702 \pm 0.007$ & $0.750 \pm 0.005$ & $0.744 \pm 0.008$ \\
      \bottomrule
    \end{tabular}
  \end{subtable}

  \vspace{1.2em}

  \textbf{Baselines (mean $\pm$ standard error)}\\[0.3em]
  \begin{tabular}{lcc}
    \toprule
    Method & Reach & Pick-place \\
    \midrule
    GAIL  & $-1.81 \pm 0.37$ & $-0.004 \pm 0.002$ \\
    AIRL  & $-1.58 \pm 0.30$ & $-0.001 \pm 0.000$ \\
    BC    & $-0.93 \pm 0.08$ & $-0.002 \pm 0.001$ \\
    GT-RL & $0.957 \pm 0.005$ & $0.903 \pm 0.011$ \\
    \bottomrule
  \end{tabular}
\end{table}

\subsection{Set Transformer Ablation}

Motivated by work showing that set functions can be universally approximated by architectures that combine element-wise encodings with simple additive pooling \citep{zaheer2017deepsets}, we tested whether we can replace our set transformer (see Appendix \ref{app:arch}) with a simple sum of trajectory encodings: $\psi = \sum_{i=1}^n \psi_i$. Specifically, we reran the demonstration-efficiency experiment using both the original variants of SRI, SGI, and SAI, and variants in which the task encoder is an additive pooling layer (Figure \ref{fig:app-arch_ablation}). The performance improvement offered by the transformer-set-transformer (TST) architecture to SRI and SAI is significant but slight for higher demonstration counts; however, prior results (not shown here) suggest that this divergence grows when the training dataset size is decreased. We suspect, though are not sure, that the exceptionally poor performance of 100-demonstration SGI with additive pooling is due to difficulties managing output magnitudes when summing large numbers of demonstration encodings, which is less of a problem for SRI and SAI (the resulting policies of which mostly remain the same regardless of model output scale in our experiments).\footnote{Note that an architecture with a \textit{mean} pooling layer need not be a universal set function approximator if the input sets have varying sizes; although that is not the case in our experiments, we avoid mean pooling for this reason.}
\begin{figure}
\centering
    \includegraphics[width=0.65\linewidth]{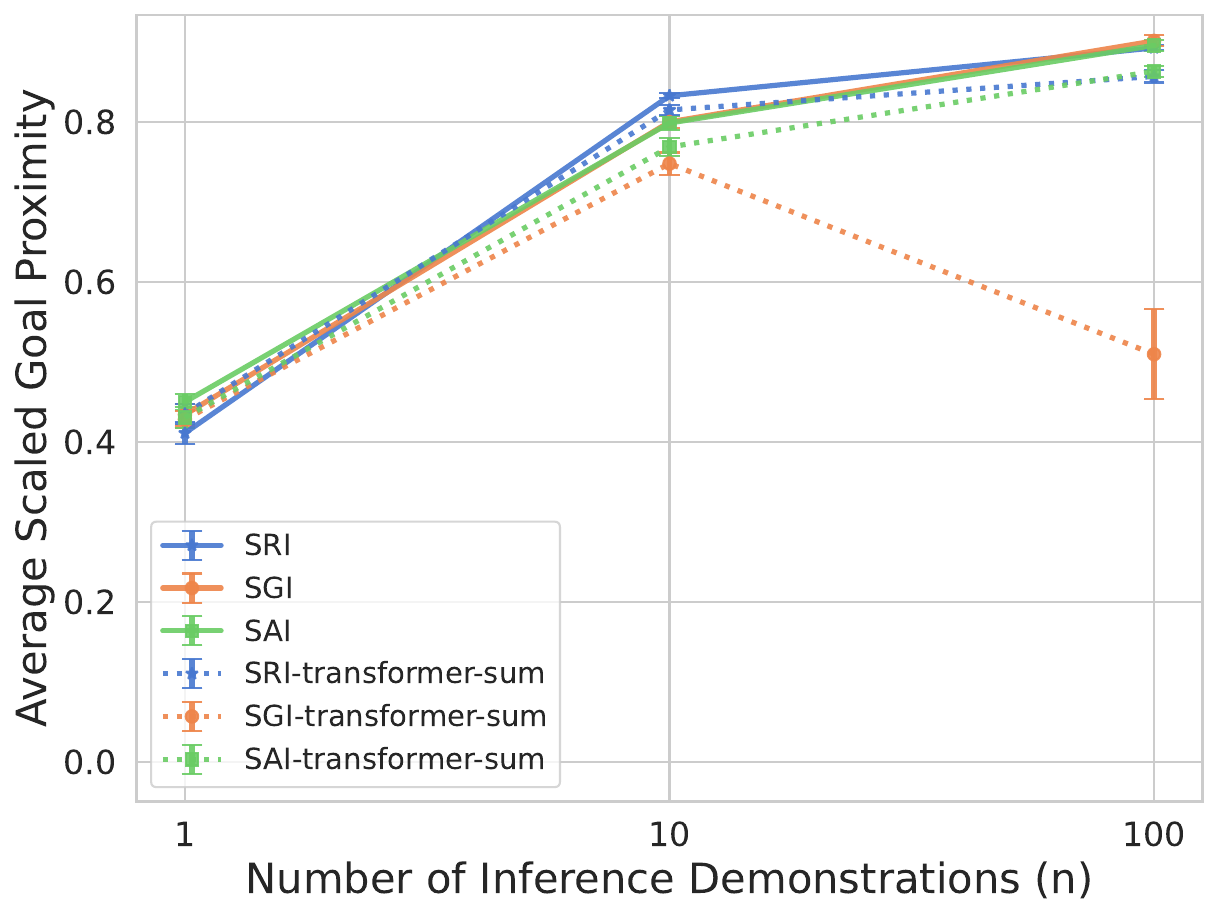}
    \captionof{figure}{%
      \textbf{Set Transformer Ablation}. Mean normalized goal proximity ($\uparrow$) in a reach task for the standard TST (transformer-set-transformer) architecture (solid line) versus an architecture which replaces the set transformer with a simple additive pooling layer (dotted line). Demonstrations are \textsc{noisy}$_{0.87}$.}
    \label{fig:app-arch_ablation}
\end{figure}

\section{Meta-World Experimental Details}
\label{app:experimentaldetails}

\subsection{Tasks}

Goals were selected randomly with $x \in [-0.3, 0.3]$, $y \in [0.4, 0.7]$, and $z \in [0.05, 0.3]$. For pick-place tasks, the object is randomly initialized per-episode for $x \in [-0.1, 0.1]$, $y \in [0.6, 0.7]$, and goals are chosen sufficiently far away to avoid initial success due to random object placement. All tasks used a horizon of 500. The default shaped reward functions were used both for SRI reward samples and ground-truth RL, except for the extra 5.0 success reward.

Note that the Meta-World reward is not computable from the default state representation; therefore, we provided SRI's state encoder with augmented states during both training and RL. Specifically, we provided SRI with: the location and starting location of the left and right pads of the gripper; the location of the ``TCP (Tool Center Point) Center''; and the initial positions of the end-effector and the object. All of these values are computable from the history, but our goal in these experiments was not to test SRI's performance in partially observable environments. Note that the learned policy did not see these augmented state dimensions. 

Unfortunately, we could not augment states for the baseline algorithms to include initial position information, as doing so provided the adversarial discriminators with a ``hack'' to distinguish demonstrations from generated behavior, but we provided them with the TCP center and pad locations.

To simplify both supervised learning and reinforcement learning, we shifted and scaled all rewards to have the range $[-3, 3]$ (after adding the extra success reward of 5.0).

\subsection{Demonstration Details}
\label{app:demonstrationdetails}

All demonstration classes were generated by oracular policies with randomly sampled reach goals. Unless otherwise specified, we provided 100 demonstrations per task for \textsc{gesture}, \textsc{noisy}, and \textsc{noisy gesture}; 1 demonstration for the deterministic \textsc{psychic}; and 10 demonstrations for \textsc{hard}.

\textsc{Noisy}$_\varepsilon$ details: the end-effector starts in the default location and reaches directly towards the goal for 150 timesteps, but with probability $\varepsilon$ each timestep of instead reaching towards a random location.

\textsc{Noisy gesture}$_\varepsilon$ details: uses random actions like \textsc{noisy}$_\varepsilon$, but otherwise identical to \textsc{gesture} (including a horizon of 50).

\textsc{Psychic}$_\alpha$ details: Given goal position $g$ and x-y origin $o_{x,y}$ = (0, 0.55), the end-effector reaches deterministically from its starting position towards target $t_{x, y} = o_{x, y} + \alpha(g_{x,y} - o_{x, y}).$

\textsc{Hard} details: the end-effector starts in a random position, then over 250 timesteps draws a 0.1-radius circle around the location of the goal mirrored through the x-y-z origin $o$ = (0, 0.55, 0.175). Note that the end-effector cannot always reach the far point of this circle, so the mirrored goal cannot be determined through averaging alone.

For \textsc{gesture} tasks, the end-effector starts in a random position satisfying $x \in [-0.4, 0.4], y\in[0.4,0.8],z\in[0.1,0.4]$.

For all reach demonstrations, the oracular policy used a speed multiplier ($p$ in the code) of 30.0, ensuring that the $(x, y, z)$ components of its actions were nearly always clipped to [-1, 1]. For all pick-place demonstrations, the same speed multiplier was used, but the clipping range was decreased to [-0.3, 0.3], as we found that the policy tended to move too quickly and drop the object otherwise.

\subsection{SRI Training Details}
\label{app:sritrainingdetails}

\textbf{Data quantity.} Except where otherwise noted, SRI received 1,280 tasks, each with 100 demonstrations and 10,000 state-reward pairs. For pick-place tasks, 80\% of states were from random reaching, while 20\% were from random pick-placing.

\textbf{State-reward sampling.} For all tasks, SRI received a randomly sampled state-reward dataset (see Section~\ref{sec:datasetconstruction}) with states produced by the robot end-effector attempting to reach to random locations on and around the table while randomly opening and closing its gripper. For pick-place tasks, SRI also received examples of states from an oracular pick-place agent bringing the object to random positions. SRI could not infer goals from such states, as they were randomly shuffled and uncorrelated with the goal of the task it was attempting to infer.

\textbf{Training.} SRI was trained for 2,000 epochs on all tasks with a batch size of 16 and learning rate of 0.0003, using the Adam optimizer \citep{kingma2015adam}. In general, the hyperparameters we used for SRI and RL were found through extensive interactive experimentation rather than structured sweeps.

As mentioned in Section \ref{sec:results}, we found hacking of SRI's rewards to be a particular problem in pick-place tasks, where a lack of model capacity forced SRI to trade off modeling the success reward on one hand and the grasping and pre-grasping shaped rewards on the other. SRI generally chose to optimize the success reward, as doing so was optimal for minimizing MSE, but this caused pick-place policies trained with the resulting reward to often fail to grasp the object at all. In order to encourage SRI to focus on modeling the non-success rewards, we linearly increased the proportion of pick-place state samples in the state-reward dataset from 5\% to 20\% over the course of training the reward model. However, the total proportion of reach and pick-place state samples per task was still 0.8 and 0.2, respectively; when necessary, we therefore sampled reach states with replacement to achieve the required total number of states per task.

\subsection{Reinforcement Learning}

For learning policies with SRI's learned reward functions, we used Truncated Quantile Critics (TQC; \citealp{kuznetsov2020controllingoverestimationbiaswithtruncatedmixtureofcontinuousdistributionalquantilecritics}) for reach tasks and Proximal Policy Optimization (PPO; \citealp{schulman2017proximalpolicyoptimizationalgorithms}) for pick-place tasks, each as implemented in Stable-Baselines3 \citep{stable-baselines3}. For reach tasks, we trained multi-task policies by conditioning the policy on the task representation $\psi$; for pick-place tasks, we trained task-specific policies (though all tasks within each trial were inferred by a single SRI model).

The reason we used PPO for pick-place tasks and TQC for reach tasks is empirical: we found that PPO learned slowly in reach tasks and that TQC often failed to learn in pick-place tasks. All RL and baseline policies were trained for 5,000,000 environment interactions for reach tasks and 10,000,000 for pick-place tasks.

All networks of all RL policies were two-layer, hidden dimension 512 MLPs. For baselines, the actor had this structure, but the critic had the same structure as SRI's state encoder: a two-layer MLP with hidden dimension 256.

Baseline hyperparameters were set to equal RL hyperparameters where appropriate (e.g., number of interactions and network size), and otherwise were taken from their respective papers as much as possible. Behavioral cloning was trained for 50 epochs.

RL algorithms used batch sizes of 128, learning rates of 0.0001, and $\gamma=0.9$. TQC used two critics, a rollout buffer size of 10,000,000, one update per transition, and a soft update coefficient of 0.005; PPO collected 2,048$*$16 transitions between updates and trained for 10 epochs per transition batch.

Note that the strangely poor behavior of the baselines in some conditions, such as AIRL in single-demonstration tasks and GAIL in noisy tasks, may in part be due to Meta-World tasks violating their assumptions. For example, GAIL assumes infinite-horizon tasks \citep{ho2016generative}. Nevertheless, as dominant imitation learning algorithms, we still believe them to be important baselines for comparison.

\subsection{Metrics and Statistics}
\label{app:metrics}
The evaluated metric is average normalized goal proximity: 1 minus the distance to the goal of the end-effector (measured at the TCP center; for reach tasks) or object (for pick-place tasks), scaled such that the initial distance is 1, and averaged across all timesteps and trials. To avoid distraction by large negative proximity values in plots, we clip average proximity of each trial in plots to a minimum of 0 (but do not clip in tables).

All methods and all settings in all experiments were run for 30 trials (including dataset resampling), and statistical uncertainty was quantified using standard error without Bonferroni correction.

\subsection{Baselines}
\label{app:baselines}
We used three single-task imitation learning baselines, all from the Imitation library \citep{gleave2022imitation}: behavioral cloning (BC; \citealp{pomerleau1988alvinn}), generative adversarial imitation learning (GAIL; \citealp{ho2016generative}), and adversarial inverse reinforcement learning (AIRL; \citealp{fu2018learning}). These baselines are single-task and are included to illustrate how optimality-assuming methods degrade with suboptimal demonstrations. We also include policies trained with ground-truth rewards (GT-RL) as a ceiling baseline.

\subsection{Additional Code Attribution}
Table~\ref{tab:assets} lists the external assets used in this work, along with their licences and versions.
\begin{table}[h]
  \caption{External assets relied on in this work.}
  \label{tab:assets}
  \small
  \renewcommand{\arraystretch}{1.2}
  \centering
  \begin{tabularx}{\textwidth}{@{} l X l l X @{}}
    \toprule
    Asset & Authorship / citation & License & Version / commit & Special terms of use \\
    \midrule
    Stable‐Baselines3 & 
    \citet{stable-baselines3}
    & MIT & 2.2.1 & None beyond MIT licence \\[2pt]

    Meta-World & 
    \citet{yu2019metaworld}
    & MIT & v2.0.0 & Requires the MuJoCo physics engine (Apache-2.0) for simulation \\[2pt]

    Set Transformer & 
    \citet{lee2019set} & MIT & \texttt{73432c6} & None beyond MIT licence \\[2pt]

    Imitation & 
    \citet{gleave2022imitation} & MIT & 1.0.1 & None beyond MIT licence \\
    
    MuJoCo & \citet{todorov2012mujoco}
    & Apache-2.0 & 2.3.7 & Open-source; no additional restrictions \\ \bottomrule
  \end{tabularx}
\end{table}

\subsection{Computational Resources}

All of our experiments were performed on a wide variety of NVIDIA GPUs, from 1080Ti to L40S and A100. VRAM was not generally a limiting factor, and we are certain that 11 GB is sufficient for the batch sizes that we used. We used 2 CPU cores for SRI/SGI/SAI model training, 16 cores for RL and baseline training, and 4-16 cores for evaluation. We used 40 GB RAM for all experiments. All model training and RL runs finished in under 24 hours each, and substantially faster on new GPUs.

\section{Experimental details for Shah et al.~(2019) gridworld benchmark}
\label{app:shah_gridworld}

To directly compare SRI with the closest prior work on reward inference from arbitrary suboptimal behavior, we reran the gridworld benchmark of \citet{shah2019feasibility} and added SRI as an additional method.
We reused the original Shah et al.\ code for the planner-learning baselines as much as possible.

\subsection{Benchmark and evaluation protocol}

Shah et al.'s benchmark uses randomly generated gridworlds with a $14\times 14$ traversable interior, seven non-zero reward cells, and several synthetic demonstrator classes including optimal, naive, sophisticated, myopic, overconfident, and underconfident agents, together with Boltzmann-noisy variants of each \citep{shah2019feasibility}.
Their evaluation metric is the discounted true return obtained by planning optimally with the inferred reward, reported as a percentage of the return obtained when planning with the ground-truth reward.

Each task is represented as a $16\times 16$ tensor with an always-walled border, corresponding to the same $14\times 14$ traversable interior used in the original benchmark.
The action space has five actions: north, south, east, west, and stay.
By default, each task contains seven non-zero reward cells, the start state is fixed at the grid center, and transition noise is $0.2$.
Non-\textsc{stay} actions incur a living reward of $-0.01$, and the discount factor is $\gamma=0.95$.
We evaluate all methods with the same percent-reward metric as in \citet{shah2019feasibility}.

\subsection{Data splits and trial structure}

Shah et al.'s original Algorithm~1 uses 8{,}000 total policy examples, split into 7{,}000 tasks with known rewards for planner training and 1{,}000 tasks on which rewards must be inferred \citep{shah2019feasibility}.
We use 5{,}000 planner-training tasks, 2{,}000 planner-validation tasks, and 1{,}000 reward-inference tasks per trial, preserving the same 7{,}000/1{,}000 overall budget while making validation explicit.
All figure results reported here use 30 trials, with one seed per trial.

For the assumed-planner baselines, we use 8{,}000 reward-inference tasks per trial, matching the no-known-reward setting of \citet{shah2019feasibility}.
For SRI, the figure uses the best validation checkpoint from each trial.

\subsection{Shah et al.\ baselines}

We compare against the same three method families shown in Shah et al.'s Figure~4: a fixed Boltzmann-planner assumption, a fixed optimal-planner assumption, and their Algorithm~1, which first trains a differentiable planner on reward-policy pairs and then infers rewards by gradient descent through that planner \citep{shah2019feasibility}.
We use the original TensorFlow implementation of Shah et al.'s planner architecture and training procedures.
Concretely, the planner is a value iteration network (VIN) with the original tabular setup.

\subsection{SRI with policy input}

The original SRI formulation conditions on a set of behavior trajectories and predicts rewards for queried states.
The benchmark of \citet{shah2019feasibility} instead exposes a full demonstrator policy on a tabular state space.
To make a direct comparison possible, we therefore use a policy-input variant of SRI.

For each task, we convert the demonstrator policy into a single structured demonstration tensor with one token per grid cell.
Each token contains the demonstrator's action probabilities at that state, together with a wall indicator and normalized spatial coordinates.
Because the reward in this benchmark is itself a tabular state-value map, the model predicts the full reward map jointly rather than evaluating one queried state at a time.

Architecturally, the reported runs use a compact U-Net \citep{ronneberger2015unet} policy encoder with base width 32.
The model reshapes the policy-input tensor into an $H\times W$ feature grid, applies two downsampling convolutional blocks, a bottleneck block, and two upsampling stages with skip connections, and then outputs one scalar reward per grid cell with a final $1\times 1$ convolution.
Each convolutional block consists of two $(3\times 3\ \mathrm{conv} + \mathrm{batchnorm} + \mathrm{LeakyReLU})$ layers.
The resulting reward map is flattened and trained with mean-squared error against the ground-truth tabular reward.

\subsection{SRI training details}

SRI uses the U-Net encoder described above, 100 training epochs, batch size 128, learning rate $3\times 10^{-4}$, and no additional regularization.

\end{document}